\documentclass{article}

\newcommand{\resreport}{
    \newcommand{\resrep}[1]{##1}
    \newcommand{\arttocl}[1]{}}

\resreport

\arttocl{\documentclass[acmtocl,acmnow]{acmtrans2m}}

\usepackage{amssymb}
\usepackage{arcs}
\usepackage{amsfonts}
\usepackage{amsmath}
\usepackage{graphicx}
\usepackage{url}
\usepackage[amsmath,thmmarks]{ntheorem}

\newcommand{\z}{z}

\newcommand{\target}{nesting}
\newcommand{\Target}{Nesting}

\newcommand{\Tnat}{\T_\mathbb{Z}}
\newcommand{\Inat}{\I_\mathbb{Z}}

\newcommand{\Ireals}{\I_\mathbb{R}}
\newcommand{\Tfol}{\T_{\text{fol}}}
\newcommand{\Tfolne}{\V_{\text{fol}}}
\newcommand{\Ifol}{\Iall}
\newcommand{\Cfol}{\C_{\text{fol}}}
\newcommand{\bsubstitution}{base substitution}
\newcommand{\tsubstitution}{\target\ substitution}
\newcommand{\bterm}{base term}
\newcommand{\tterm}{{\target} term}
\newcommand{\bvariable}{base variable}
\newcommand{\tvariable}{\target\ variable}
\newcommand{\Cnat}{\C_\mathbb{Z}}
\newcommand{\Cnatok}{\C_\mathbb{Z}'}
\newcommand{\Tarrays}{\T_\mathbb{A}}
\newcommand{\Iarrays}{\I_\mathbb{A}}
\newcommand{\Carrays}{\C_\mathbb{A}}
\newcommand{\Cearraysint}{\C_\mathbb{A}^\mathbb{Z}}
\newcommand{\Carraysint}{\C_\mathbb{Z}}
\newcommand{\Cearrays}{\C_\mathbb{A}}

\newcommand{\Tearrays}{\V_\mathbb{A}}

\newcommand{\Tearraysint}{{\V_\mathbb{A}^\mathbb{Z}}}

\newcommand{\Tarraysintintint}{\T_{\mathbb{Z}_3}}

\newcommand{\Tearraysreals}{\V_\mathbb{A}^\mathbb{R}}
\newcommand{\Cearraysreals}{\C_\mathbb{A}^\mathbb{R}}

\newcommand{\Tnatok}{\U_{\mathbb{Z}}}

\theoremseparator{.}
\newtheorem{theorem}{Theorem}
\newtheorem{lemma}[theorem]{Lemma}

\newtheorem{proposition}[theorem]{Proposition}
{\theorembodyfont{\upshape}
  \theoremsymbol{\ensuremath{\Diamond}}
  \newtheorem{definition}[theorem]{Definition}
  
  \theoremsymbol{}
  \newtheorem{example}[theorem]{Example}
  \theoremstyle{nonumberplain}

}
\theoremheaderfont{\sc}\theorembodyfont{\upshape}
\theoremstyle{nonumberplain}
\theoremsymbol{\rule{1ex}{1ex}}
\resrep{\newtheorem{proof}{Proof}}

\newcommand{\commentthis}[1]{}

\newcommand{\isdef}{\stackrel{\mbox{\tiny def}}{=}}
\newcommand{\subseteqdef}{\stackrel{\mbox{\tiny def}}{\subseteq}}
\newcommand{\indef}{\stackrel{\mbox{\tiny def}}{\in}}

\newcommand{\true}{\texttt{true}}
\newcommand{\false}{\texttt{false}}
\newcommand{\dom}{\mbox{\it dom}}
\newcommand{\Dom}[1]{\mbox{\it dom}(#1)} 
\newcommand{\Ran}[1]{\mbox{\it cod}(#1)}
\newcommand{\id}{\mbox{\it id}}
\newcommand{\set}[1]{\{#1\}}

\newcommand{\setof}[2]{\{#1\,|\:#2\}}
\newcommand{\Setof}[2]{\Bigl\{#1\,|\:#2\Bigr\}}

\newcommand{\dN}{\mathbb{N}}

\newcommand{\asort}{{\tt s}}
\newcommand{\snat}{{\tt int}}
\newcommand{\sreal}{{\tt real}}
\newcommand{\sarrays}{{\tt array}}
\newcommand{\sarray}{{\tt array}}
\newcommand{\sbool}{{\tt bool}}
\newcommand{\selem}{{\tt elem}}

\newcommand{\sigterms}[2]{\mathrm{T}_{#1}(#2)}
\newcommand{\gt}[1]{\mathrm{T}_{#1}}

\newcommand{\iseq}{\simeq} 

\newcommand{\interp}[2]{#1^{#2}}
\newcommand{\inter}[2]{#1^{#2}}
\newcommand{\aclass}{{\mathfrak C}}
\newcommand{\C}{{\aclass}}

 \newcommand{\expr}{\mathcal{E}}
\newcommand{\complete}{complete}

\newcommand{\preserving}{$\base$-invariant}
\newcommand{\preservation}{$\base$-invariance}
\newcommand{\am}{\alpha}

\newcommand{\theory}{specification}
\newcommand{\theories}{specifications}
\newcommand{\Theory}{Specification}
\newcommand{\Theories}{Specifications}

\newcommand{\btheory}{base specification}
\newcommand{\ttheory}{{\target} specification}

\newcommand{\valueof}[2]{[#1]_{#2}}

\newcommand{\T}{\mathcal{A}} 
\newcommand{\U}{\mathcal{B}} 
\newcommand{\V}{\mathcal{N}} 

\newcommand{\I}{\mathcal{I}} 
\newcommand{\X}{\mathcal{X}} 
\newcommand{\F}{\mathcal{F}} 
 \newcommand{\inst}{\Theta}
 \newcommand{\allsorts}{\allSorts}
 \newcommand{\allSorts}{{\tt S}}
\newcommand{\rnk}{\mathrm{rnk}}
\newcommand{\bool}{\texttt{bool}}

\newcommand{\vars}{{\cal X}} 
\newcommand{\IntA}{I}

\newcommand{\hcombin}[2]{#2[#1]}
\newcommand{\hcombinproc}[2]{#2[#1]}

\newcommand{\nf}[1]{\gamma(#1)}
\newcommand{\nff}[1]{#1\hspace*{-0.1cm}\downarrow_{\gamma}}

\newcommand{\falsifyingsubs}[2]{\Phi_{#1}({#2})}

\newcommand{\var}{\text{Var}} 
\newcommand{\Var}[1]{\var(#1)}

 \newcommand{\arrays}{\texttt{array}}
 \newcommand{\elem}{\texttt{elem}}
 \newcommand{\indices}{\texttt{ind}}
 \newcommand{\select}{\mathrm{select}}
 \newcommand{\store}{\mathrm{store}}
 \newcommand{\nat}{\texttt{nat}}
\newcommand{\Ax}{{\cal A}}
\newcommand{\Gr}{{\cal G}}

\arttocl{\markboth{M. Echenim and N. Peltier}{Instantiation schemes for nested theories}}

\title{Instantiation Schemes for Nested Theories
\resrep{\\ (research report)}}

\arttocl{\author{MNACHO ECHENIM\\ Ensimag/Laboratory of Informatics of Grenoble \and NICOLAS PELTIER\\ CNRS/Laboratory of Informatics of Grenoble}}
\resrep{\author{Mnacho Echenim\\ Ensimag/Laboratory of Informatics of Grenoble \and Nicolas Peltier\\ CNRS/Laboratory of Informatics of Grenoble}}

\resrep{\date{July 2011}}

\resrep{
\begin{document} \maketitle}

\begin{abstract}
This paper investigates under which conditions instantiation-based proof
 procedures can be combined in a \emph{nested} way, in order to mechanically
 construct new instantiation
 procedures for richer theories. Interesting applications in the field of verification
 are emphasized, particularly for handling extensions of the theory of arrays.
\end{abstract}

\arttocl{


}

\arttocl{\begin{document}}

\arttocl{
\begin{bottomstuff}
\end{bottomstuff}
}

\arttocl{\maketitle}

\section{Introduction}


Proving the satisfiability or unsatisfiability of a first-order
formula (possibly modulo some background theory) is an essential
problem in computer science -- in particular for the automatic
verification of complex systems, and \emph{instantiation schemes} can
be used for this purpose.  Such schemes can be viewed as  functions
$\inst$ that map a set of formul{\ae} (or clauses) $S$ to a set of
ground (i.e. without variable) instances $\inst(S)$ of $S$. An
instantiation scheme $\inst$
is \emph{refutationally complete} if for all sets of clauses $S$,
$\inst(S)$ is satisfiable exactly when $S$ is. Examples of
refutationally complete instantiation schemes include
\cite{LP92,PZ00,GanzingerKorovin-03-lics,Baumgartner:Tinelli:ModelEvolutionCalculus:CADE:2003}.
It is clear that an instantiation scheme that is refutationally
complete does not always terminate, as $\inst(S)$ may be infinite, but
schemes that are both complete and terminating can be defined for
specific classes of clause sets, that are thus decidable. A trivial
and well-known example is the Bernays-Sch\"{o}nfinkel class (i.e. the
class of purely universal formul{\ae} without function symbols of
arity distinct from $0$, see, e.g., \cite{DG79}), since in this case the
set of ground instances is finite. Other examples include the class of
\emph{stratified} clause sets \cite{Abadi2010153} and many classes of
clause sets of the form $\Gr \cup \Ax$, where $\Gr$ is a set of
\emph{ground} formul{\ae} and $\Ax$ is the set of axioms of a specific
theory, for instance the theory of arrays \cite{Bradleybook}. In this
last case, of course, only the axioms in $\Ax$ need to be
instantiated.

Instantiation schemes can also be defined for specific theories for which decision procedures exist. Then, the theory is not axiomati{\z}ed, but directly handled by an external prover -- used as a ``black box''. In this case, the instantiation procedure should preserve the validity of the formula modulo the considered theory. Such procedures are appealing, because it is usually much easier to check the  validity of a ground set than that of a non-ground set (see for instance \cite{Bradley}).


 \newcommand{\sorted}{\text{\it sorted}}

 Frequently, one has to handle heterogeneous problems, defined on
 complex theories for which no instantiation procedure exists. Such
 theories are frequently obtained by combining simpler theories.  For
 instance the theory describing a data-structure (arrays, list, etc.)
 may be combined with the theory modeling the elements it contains
 (e.g., integers).
 Most systems rely on the Nelson-Oppen method (and its numerous
 refinements) to reason on combination of theories. This scheme allows
 one -- under certain conditions -- to combine independent decision
 procedures (see, e.g., \cite{Tinelli96anew}), but it is of no use for
 reasoning on theories that include axioms containing function or
 predicate symbols from both theories. As an example, consider the
 following formula: $$\forall i,j: \nat,\ i \leq j \Rightarrow
 \select(t,i) \leq \select(t,j),$$ that states that an array $t$ is
 sorted.  This formula uses symbols from the theory of integers (the
 predicate $\leq$) and from the theory of arrays (the function
 $\select$, which returns the value stored in a certain array at a
 certain index).

 In this paper, we show how to construct \emph{automatically}
 instantiation schemes for such axioms, by combining existing
 instantiation schemes.  More precisely, from two complete
 instantiation procedures $\Theta_{\mathbb{N}}$ and
 $\Theta_{\mathbb{A}}$ for the theory of integers and for the theory
 of arrays respectively, we construct a new procedure $\Theta$ which
 is able to handle a particular class of ``mixed'' axioms, containing
 function symbols from both theories (including for instance the
 axioms for sorted arrays and many others).  $\Theta$ will be complete
 and terminating if both $\Theta_{\mathbb{N}}$ and
 $\Theta_{\mathbb{A}}$ are (as proven in Section \ref{sect:complete}).  This approach is not restricted to
 specific theories such as $\Theta_{\mathbb{N}}$ and
 $\Theta_{\mathbb{A}}$; on the contrary it is \emph{generic} and
 applies to a wide range of theories  and some examples are provided in
 Section \ref{sect:appl}. The conditions that must be satisfied by
 the considered theories and by their instantiation procedures are
 very precisely identified (see Section \ref{sect:conditions}).

\subsection*{Comparison with Related Work}

There is an extensive amount of work on the combination of (usually
disjoint) theories, using mainly refinements or extensions of the
Nelson-Oppen method (see, e.g.,
\cite{Tinelli96anew,DBLP:journals/amai/BruttomessoCFGS09}). For
instance, \cite{DBLP:conf/frocos/Fontaine09} shows that many decidable
fragments of first-order logic can be combined with any disjoint
theory, even if these fragments do not fulfill the stable infiniteness
condition in general. A related result is presented in
\cite{DBLP:conf/lpar/FontaineRZ04} for the theory of lists (with a
length function).  However, these results do not apply to non-disjoint
theories as the ones we consider in this paper, and they cannot handle
\emph{nested} combinations of \emph{arbitrary} theories.

Reasoning on the combination of theories with mixed axioms has been
recogni{\z}ed as an important problem and numerous solutions have been
proposed in many specific cases.  Most existing work focuses on testing
the satisfiability problem of \emph{ground} formul{\ae} in
combinations or extensions of existing theories.  In contrast, our
method aims at reducing non-ground satisfiability to ground
satisfiability tests, via instantiation.

For instance, \cite{Bradley,Bradleybook} define a decision procedure
for extensions of the theory of arrays with integer elements, which is
able to handle axioms such as the one above for sorted arrays. As we
shall see in Section \ref{sect:appl}, our approach, when applied to
these particular theories, permits to handle a strictly more
expressive class of quantified formul{\ae}.

\cite{springerlink:10.1007/s10472-007-9078-x} focuses on arrays with
integer indices and devises a method to combine existing decision
procedures (for Presburger arithmetic and for the theory of arrays).
This method is able to handle some important specific features of
arrays such as sortedness or array dimension.  Similarly to our
approach, theirs is based on an instantiation of the axioms.  As we
shall see, some of its features can be tackled with our method and
others (such as Injectivity) are out of its scope.  However, our
method is \emph{generic} in the sense that it applies to a wide class
of theories and axioms (in particular, it applies to axioms that are
not considered in \cite{springerlink:10.1007/s10472-007-9078-x}). It
is essentially syntactic, whereas that of
\cite{springerlink:10.1007/s10472-007-9078-x} is more of a semantic
nature.

A logic devoted to reasoning with arrays of integers is presented is
\cite{DBLP:conf/fossacs/HabermehlIV08} and the decidability of the
satisfiability problem is established by reduction to the emptiness
problem for counter automata.  In Section \ref{sect:appl} we shall
show that the expressive power of this logic is again incomparable
with the one we obtain with our approach.

\cite{dMGe} proposes an instantiation scheme for sets of clauses
possibly containing arithmetic literals, which can handle some of the
axioms we consider. However termination is not guaranteed for this
scheme, in contrast to ours.


\newcommand{\basetheory}{{\cal B}}
\newcommand{\nonbasetheory}{{\cal N}}

Slightly closer to our approach is the work described in
\cite{DBLP:conf/cade/Sofronie-Stokkermans05,DBLP:conf/cade/Sofronie-Stokkermans10},
which defines the notion of the \emph{(stably) local extension} of a theory
and shows that the satisfiability problem in a (stably) local
extension of a theory $\T$ can be reduced to a mere satisfiability
test in $\T$.  The notion of a local extension is a generali{\z}ation of
the notion of a local theory \cite{DBLP:conf/lics/Ganzinger01}. The
idea is that, for testing the satisfiability of a ground formula $\Gr$
in the local extension of a theory, it is sufficient to instantiate
the variables occurring in the new axioms by ground terms occurring
either in $\Gr$ or in the axioms.  This condition holds for numerous
useful extensions of base theories, including for instance extensions
with free functions, with selector functions for an injective
constructor, with monotone functions over integers or reals etc.  Our
approach departs from these results because our goal is not to extend
basic theories, but rather to combine existing instantiation
procedures.  Note also that the notion of a local extension is a
semantic one, and that this property must be established separately
for every considered extension. In our approach we define conditions
on the theories ensuring that they can be safely combined. These
conditions can be tested \emph{once and for all} for each theory, and then any
combination is allowed. The extensions we consider in this paper are
not necessarily local thus do not fall under the scope of the method in
\cite{DBLP:conf/cade/Sofronie-Stokkermans05,DBLP:conf/cade/Sofronie-Stokkermans10}.
However, an important restriction of our approach compared to
\cite{DBLP:conf/cade/Sofronie-Stokkermans05,DBLP:conf/cade/Sofronie-Stokkermans10}
is that the theories must be combined in a \emph{hierarchic} way:
intuitively there can be function symbols mapping elements of the
first theory $\basetheory$ (the ``base'' theory) to elements of the
second one $\nonbasetheory$ (the ``\target'' theory), but no function
symbols are allowed from $\nonbasetheory$ to $\basetheory$.




Extensions of the superposition calculus \cite{BG94} have been proposed to handle first-order extensions of a base theory (see for example \cite{BachmairGW94,AlthausKW09}). The superposition calculus is used to reason on the generic part of the formul{\ae} whereas the theory-specific part is handled by an external prover. These proof procedures can be used to reason on some the formul{\ae} we consider in the present paper. However, we are not aware of any termination result for these approaches (even completeness requires additional restrictions that are not always satisfied in practice).
Our approach uses an instantiation-based approach instead of superposition, and ensures that termination is preserved by the combination, at the cost of much stronger syntactic restrictions on the considered formul{\ae}.

\subsection*{Organi{\z}ation of the Paper}

\newcommand{\Gclauses}{$\omega$-Clauses}
\newcommand{\gclauses}{$\omega$-clauses}
\newcommand{\gclause}{$\omega$-clause}

The rest of the paper is structured as follows.  Section
\ref{sect:prel} contains general definitions and notations used
throughout the present work. Most of them are standard, but some are
more particular, such as the notions of {\gclauses} or {\theories}. Section \ref{sect:combine} describes our procedure for the
nested combination of instantiation schemes, and introduces conditions
to ensure that completeness is preserved. Section \ref{sect:appl}
shows some interesting applications of these results for theories that
are particularly useful in the field of verification (especially for
extensions of the theory of arrays). Section \ref{sect:disc} concludes
the paper and gives some lines of future work.

\section{Preliminaries}

\label{sect:prel}

In this section, we first briefly review usual notions and notations
about first-order clausal logic.  Then we introduce the rather
nonstandard notion of an {\em \gclause}\ (a clause with
infinitely many literals).  We define the notion of \emph{\theories}\
and provide some examples showing how usual theories such as those for
integers or arrays can be encoded. Finally we introduce the notion of
 instantiation methods.


\subsection{Syntax}

\label{sect:syntax}

\newcommand{\base}{\allsorts_{\bind}}
\newcommand{\baseterm}{\gt{\bind}}
\newcommand{\nonbase}{\allsorts_{\nbind}}
\newcommand{\bind}{B}
\newcommand{\nbind}{N}
\newcommand{\basevar}{\X_\bind}
\newcommand{\nonbasevar}{\X_\nbind}
\newcommand{\clbase}{{\Omega}_{\bind}}
\newcommand{\clnonbase}{{\Omega}_{\nbind}}
  \newcommand{\sign}{\Sigma}

  Let $\allSorts$ be a set of \emph{sort symbols} and $\F$ be a set of
  \emph{function symbols} together with a \emph{ranking function}
  $\rnk: \F \rightarrow \allSorts^* \times \allSorts$.  For every $f
  \in \F$, we write $f: \asort_1\times\cdots \times \asort_n
  \rightarrow \asort$ if $\rnk(f) = \asort_1,\ldots,\asort_n,\asort$.
  If $n=0$ then $f$ is a \emph{constant symbol of sort $\asort$}. We
  assume that $\F$ contains at least one constant symbol of each sort.
  To every sort $\asort \in \allSorts$ is associated a countably
  infinite set $\X_\asort$ of \emph{variables of sort $\asort$}, such
  that these sets are pairwise disjoint.  $\X = \bigcup_{s\in
    \allsorts} \X_\asort$ denotes the whole set of variables.  For
  every $\asort \in \allSorts$, the \emph{set of terms of sort
    $\asort$} is denoted by $\sigterms{\asort}{\X}$ and built
  inductively as usual on $\X$ and $\F$:
\begin{itemize}
\item{$\X_\asort \subseteqdef \sigterms{\asort}{\X}$.}
 \item{If $f: \asort_1\times\ldots\times\asort_n \rightarrow \asort$
     and for all $i \in [1,n], t_i \in \sigterms{\asort_i}{\X}$ then $f(t_1,\ldots,t_n) \indef \sigterms{\asort}{\X}$.}
 \end{itemize}
  The \emph{set of terms} is defined by $\sigterms{}{\X}
\isdef \bigcup_{\asort \in \allSorts} \sigterms{\asort}{\X}$.

 An \emph{atom} is an equality $t \iseq s$ between
terms of the same sort.
A \emph{literal} is either an
atom or the negation of an atom (written $t \not \iseq s$).
If $L$ is a literal, then $L^c$ denotes its complementary: $(t \iseq s)^c \isdef (t\not \iseq s)$ and $(t \not \iseq s)^c \isdef (t \iseq s)$.
A \emph{clause} is a finite set (written as a disjunction) of
literals.  
We assume that $\allSorts$ contains a sort $\sbool$ and that $\F$ contains a constant symbol $\true$ of sort $\sbool$.
For readability, atoms of the form $p \iseq \true$ will be simply denoted by $p$ (thus we write, e.g., $a \leq 2$ instead of $(a \leq 2) \iseq \true$). An atom is \emph{equational} iff it is of the form $t \iseq s$ where $t,s \not = \true$.

The set of variables occurring in an expression (term, atom, literal
or clause) $\expr$ is denoted by $\var(\expr)$. $\expr$ is
\emph{ground} iff $\var(\expr) = \emptyset$.  The set of ground terms
of sort $\asort$ is denoted by $\gt{\asort}$ and the set of ground
terms by $\gt{} \isdef \bigcup_{\asort \in \allsorts} \gt{\asort}$.

A
\emph{substitution} is a function that maps every variable to a term
of the same sort. The image of a variable $x$ by a substitution
$\sigma$ is denoted by $x\sigma$. The \emph{domain} of a substitution
$\sigma$ is the set\footnote{for technical convenience we do
  \emph{not} assume that $\dom(\sigma)$ is finite.} $\Dom{\sigma} \isdef
\setof{x \in \X}{x\sigma \neq x}$, and its \emph{codomain}
$\Ran{\sigma}$ is the set of elements the variables in the domain are
mapped to. Substitutions are extended to terms, atoms, literals and
clauses as usual: $f(t_1,\ldots,t_n)\sigma \isdef
f(t_1\sigma,\ldots,t_n\sigma)$, $(t \iseq s)\sigma \isdef (t\sigma
\iseq s\sigma)$, $(\neg L)\sigma \isdef \neg (L\sigma)$ and
$(\bigvee_{i=1}^n L_i)\sigma \isdef \bigvee_{i=1}^n L_i\sigma$.  A
substitution $\sigma$ is \emph{ground} if $\forall x \in
\Dom{\sigma}$, $\var(x\sigma) = \emptyset$. A \emph{ground instance}
of an expression $\expr$ is an expression of the form $\expr\sigma$,
where $\sigma$ is a ground substitution of domain $\var(\expr)$.
\begin{definition}
A
substitution $\sigma$ is \emph{pure} iff for all $x \in \vars$,
$x\sigma \in \vars$. In this case, for any term $t$, $t\sigma$ is a
\emph{pure instance} of $t$. A
substitution $\sigma$ is a \emph{renaming} if it is pure and injective.
\end{definition}

A substitution $\sigma$ is a \emph{unifier} of a set of pairs $\{
(t_i,s_i) \mid i \in [1,n] \}$ iff $\forall i \in [1,n], t_i\sigma =
s_i\sigma$. It is well-known that all unifiable sets have a most
general unifier (mgu), which is unique up to a renaming.

\subsection{Semantics}

\label{sect:semantics}


\newcommand{\relin}[3]{#1 =_{#3} #2}
\newcommand{\notrelin}[3]{#1 =_{#3} #2}
\newcommand{\eqcl}[2]{[#1]_{#2}}
\newcommand{\domof}[1]{D^{#1}}
\newcommand{\domofsort}[2]{#1^{#2}}
\newcommand{\ff}{\gamma}
\newcommand{\fff}{\Gamma}
\newcommand{\Imapping}[1]{$\base$-mapping}
\newcommand{\simwith}[1]{\sim_{#1}}
\newcommand{\iffdef}{$\textrm{iff}$ }

An \emph{interpretation} $I$ is a
function mapping:
\begin{itemize}
\item{Every sort symbol $\asort \in \allsorts$ to a nonempty set $\domofsort{\asort}{I}$.}
\item{Every function symbol $f: \asort_1\times\ldots\times\asort_n \rightarrow \asort \in \F$ to a function $\inter{f}{I}: \domofsort{\asort_1}{I}\times\ldots\times \domofsort{\asort_n}{I} \rightarrow  \domofsort{\asort}{I}$.}
\end{itemize}
$\domof{I}$ denotes the domain of $I$, i.e., the set $\bigcup_{\asort
  \in \allsorts} \domofsort{\asort}{I}$.  As usual, the valuation function
$\expr \mapsto \valueof{\expr}{I}$ maps every ground expression
$\expr$ to a value defined as follows:
\begin{itemize}
\item $\valueof{f(t_1,\ldots,t_n)}{I}
\isdef \inter{f}{I}(\valueof{t_1}{I},\ldots,\valueof{t_n}{I})$,
\item $\valueof{t \iseq s}{\IntA} = \true$ \iffdef $\valueof{t}{\IntA}
  = \valueof{s}{\IntA}$,
\item $\valueof{t \not \iseq s}{\IntA} = \true$
  \iffdef $\valueof{t \iseq s}{\IntA} \not = \true$,
\item $\valueof{\bigvee_{i=1}^n L_i}{\IntA} \isdef \true$ \iffdef
  $\exists i \in [1,n], \valueof{L_i}{\IntA} = \true$.
\end{itemize}
An $\F$-interpretation $I$ \emph{satisfies} an
$\F$-clause $C$ if for every ground instance $C\sigma$ of $C$
we have $\valueof{C\sigma}{\IntA} = \true$.
A set of $\F$-clauses $S$ is \emph{satisfied} by $I$ if $I$ satisfies every clause in $S$.
If this is the case, then $I$ is a \emph{model} of $S$ and we write $I
\models S$.  A set of clauses $S$ is \emph{satisfiable} if it has a model;
two sets of
clauses are \emph{equisatisfiable} if one is satisfiable
exactly when the other is satisfiable.

In the sequel, we restrict ourselves, w.l.o.g., to interpretations such that, for every $\asort \in \allSorts$,
$\domofsort{\asort}{I} = \{ \valueof{t}{I} \mid t \in \gt{\asort} \}$.

\subsection{\Gclauses}

\label{sect:gcl}

For technical convenience, we extend the usual notion of a clause by allowing infinite disjunction of literals:

\begin{definition}
\label{def:gcl}
An \emph{\gclause} is a possibly infinite set of literals.
\end{definition}
\newcommand{\emb}{\trianglelefteq} The notion of instance extends
straightforwardly to \gclauses: if $C$ is an \gclause\ then $C\sigma$
denotes the \gclause\ $\{ L\sigma \mid L \in C \}$ (recall that the
domain of $\sigma$ may be infinite).  Similarly, the semantics of
\gclauses\ is identical to that of standard clauses: if $C$ is a
ground \gclause, then $\valueof{C}{\IntA} \isdef \true$ iff there
exists an $L \in C$ such that $\valueof{L}{\IntA} = \true$. If $C$ is
a non-ground \gclause, then $\IntA \models C$ iff for every ground
substitution of domain $\var(C)$, $\valueof{C\sigma}{\IntA} =
\true$.  The notions of satisfiability, models etc. are extended
accordingly.  If $S,S'$ are two sets of \gclauses, we write $S \emb
S'$ if for every clause $C' \in S'$ there exists a clause $C \in S$
such that $C \subseteq C'$.

\begin{proposition}
\label{prop:logc}
If $S \emb S'$ then $S'$ is a logical consequence of $S$.
\end{proposition}

Of course, most of the usual properties of first-order logic such as
semi-decidability or compactness \emph{fail} if \gclauses\ are
considered. For instance, if $C$ stands for the {\gclause} $\setof{b
  \iseq f^i(a)}{i \in \dN}$ and $D_j \isdef \set{b \not\iseq f^j(a)}$ for
$j \in \dN$, then $S \isdef \setof{D_j}{j\in \dN} \cup \set{C}$ is
unsatisfiable, although every finite subset of $S$ is satisfiable.

\subsection{\Theories}

\label{sect:theories}

Usually, theories are defined by sets of axioms and are closed under
logical consequence. In our setting, we will restrict either the class
of interpretations (e.g., by fixing the interpretation of a sort
$\snat$ to the natural numbers) or the class of clause sets (e.g., by
considering only clause sets belonging to some decidable fragments or
containing certain axioms).  This is why we introduce the (slightly unusual)
notion of \emph{\theories}, of which we provide examples in the
following section:



\begin{definition}
  A \emph{\theory} $\T$ is a pair $(\I,\C)$, where $\I$ is a set of
  interpretations and $\C$ is a class of clause sets.
  A clause set $S \in \C$ is \emph{$\T$-satisfiable} if there exists
  an $\IntA \in \I$ such that $\IntA \models S$. $S$ and $S'$ are
  \emph{$\T$-equisatisfiable} if they are both $\T$-satisfiable or
  both $\T$-unsatisfiable.  We write $S \models^{\T} S'$ iff every
  $\T$-model of $S$ is also an $\T$-model of $S'$.
  \end{definition}

  For the sake of readability, if $\T$ is clear from the context, we
  will say that a set of clauses is satisfiable, instead of $\T$-satisfiable.
  We write $(\I,\C) \subseteq (\I',\C')$ iff $\I = \I'$ and  $\C \subseteq \C'$.
  By a slight abuse of language, we say that \emph{$C$ occurs in $\T$} if there exists $S \in \C$ such that $C \in S$.
\newcommand{\Iall}{\I_{\text{fol}}}

In many cases, $\I$ is simply the set of all interpretations, which we denote by $\Iall$.
But our results also apply to domain-specific instantiation schemes
such as those for Presburger arithmetic.
Of course, restricting the form of the clause sets in $\C$ is
necessary in many cases for defining instantiation schemes that are
both terminating and refutationally complete. That is why we do not
assume that $\C$ contains every clause set.  Note that axioms may be
included in $\C$. We shall simply assume that $\C$ is closed
  under inclusion and ground instantiations, i.e., for all $S \in \C$
  if $S'\subseteq S$ and $S''$ only contains ground instances of
  clauses in $S$, then $S',S'' \in \C$. All the classes of clause sets
considered in this paper satisfy these requirements.

\newcommand{\gdefin}{$\omega$-definable}
\newcommand{\axof}[1]{\text{Ax}(#1)}

We shall restrict ourselves to a particular class of \theories: those
with a set of interpretations that can be defined by a set of
\gclauses.
\begin{definition}\label{def:gdefin}
  A \theory\ $\T = (\I,\C)$ is \emph{\gdefin}\ iff there exists a
  (possibly infinite) set of \gclauses\ $\axof{\I}$ such that $\I = \{
  I \mid I \models \axof{\I} \}$.
\end{definition}
From now on, we assume that all the considered \theories\ are \gdefin.

\subsection{Examples}

\label{sect:ex}

\begin{example}
\label{ex:fol}
The \theory\ of first-order logic is defined by $\Tfol \isdef (\Ifol,\Cfol)$ where:
\begin{itemize}
\item{$\Ifol$ is the set of all interpretations (i.e. $\axof{\Ifol} \isdef \emptyset$).}
\item{$\Cfol$ is the set of all clause sets on the considered signature.}
\end{itemize}
\end{example}

\begin{example}\label{ex:presb}
The \theory\ of Presburger arithmetic is defined as follows: $\Tnat \isdef (\Inat,\Cnat)$ where:
\begin{itemize}
\item{$\axof{\Inat}$ contains the domain axiom:
$\bigvee_{k \in \mathbb{N}} (x\iseq s^k(0) \vee x \iseq -s^k(0))$ and the usual axioms for the function symbols $0: \snat$, $-: \snat \rightarrow \snat$, $s: \snat \rightarrow \snat$, $p: \snat \rightarrow \snat$, $+: \snat \times \snat\rightarrow \snat$, and for the predicate symbols $\iseq_k : \snat \times \snat\rightarrow \sbool$ (for every $k \in \mathbb{N}$) $\leq: \snat \times \snat\rightarrow \sbool$ and $<: \snat \times \snat\rightarrow \sbool$:
\[
\begin{tabular}{cc}
$0+x \iseq x$ &
$s(x)+y \iseq s(x+y)$ \\
$p(x)+y \iseq p(x+y)$ &
$p(s(x)) \iseq x$ \\
$s(p(x)) \iseq x$ &
$s^k(0) \iseq_k 0$ \\
$-0 \iseq 0$ &
$-s(x) \iseq p(-x)$ \\
$-p(x) \iseq s(-x)$ &
$x \not \iseq_k y \vee s^k(x) \iseq_k y$ \\
$x \not \iseq_k y \vee p^k(x) \iseq_k y$ &
$x < y \Leftrightarrow s(x) < s(y)$ \\
$x \not < y \vee x < s(y)$ &
$x \leq y \Leftrightarrow (x < y \vee x \iseq y)$ \\
$x < s(x)$ \\
\end{tabular}
\]
$\iseq_k$ denotes equality modulo $k$ (which will be used in Section
\ref{sect:pa}); $x,y$ denote variables of sort $\snat$ and $k$ is any
natural number. Note that the domain axiom is an infinite \gclause, while
the other axioms can be viewed as standard clauses.}
\item{$\Cnat$ is the class of clause sets built on the set of function
    symbols $0:\snat,s:\snat \rightarrow \snat,p:\snat\rightarrow
    \snat$ and on the previous set of predicate symbols.}
\end{itemize}
In the sequel, the terms $s^k(0)$ and $p^k(0)$ will be written $k$ and $-k$
respectively.
\end{example}

\begin{example}
The \theory\ of arrays is $\Tarrays \isdef (\Iarrays,\Carrays)$ where:
\begin{itemize}
\item{$\axof{\Iarrays} \isdef \{ \select(\store(x,z,v),z) \iseq v,\ z'
\iseq z \vee \select(\store(x,z,v),z') \iseq \select(x,z') \}$, where
$\select: \arrays\times \indices \rightarrow \elem$ and
$\store: \arrays \times \indices \times \elem \rightarrow \arrays$ ($x$ is a variable of sort $\arrays$, $z,z'$ are variables of sort $\indices$ and $v$ is a variable of sort $\elem$).}
\item{$\Carrays$ is the class of ground clause sets built on $\select$, $\store$ and a set of constant symbols.}
\end{itemize}
\end{example}

It should be noted that reals can be also handled by using any axiomatization of real closed fields.

\subsection{Instantiation Procedures}

An instantiation procedure is a function that reduces the $\T$-satisfiability problem for any set of $\T$-clauses to that of
a (possibly infinite) set of \emph{ground}  $\T$-clauses.

\newcommand{\terminating}{terminating}

\begin{definition}
  Let $\T = (\I,\C)$ be a \theory.  An \emph{instantiation procedure
    for $\T$} is a function $\Theta$ from $\C$ to $\C$ such that for
  every $S \in \C$, $\Theta(S)$ is a set of ground instances of
  clauses in $S$.  $\Theta$ is \emph{\complete}\ for $\T$ if for every
  $S \in \C$, $S$ and $\Theta(S)$ are $\T$-equisatisfiable. It is
  \emph{\terminating}\ if $\Theta(S)$ is finite for every $S \in \C$.
\end{definition}

If $\Theta$ is complete and terminating, and if there exists a
decision procedure for checking whether a ground (finite) clause set is
satisfiable in $\I$, then the $\T$-satisfiability problem is clearly decidable.
Several examples of complete 
instantiation procedures are available in the literature \cite{PZ00,GanzingerKorovin-03-lics,Baumgartner:Tinelli:ModelEvolutionCalculus:CADE:2003,dMGe,DBLP:journals/cj/LoosW93,Abadi2010153,Bradley,EP10c,EP10a}. 
Our goal in this paper is to provide a general mechanism for constructing new complete instantiation procedures by combining existing ones.


\newcommand{\su}{\text{succ}}

\section{Nested Combination of \Theories}

\label{sect:combine}

\subsection{Definition}

\label{sect:bdef}

Theories are usually combined by considering their (in general
disjoint) union. Decision procedures for disjoint theories can be
combined (under certain conditions) by different methods, including
the Nelson-Oppen method \cite{Tinelli96anew} or its refinements.  In
this section we consider a different way of combining \theories.  The
idea is to combine them in a ``hierarchic'' way, i.e., by considering
the formul{\ae} of the first \theory\ as constraints on the formul{\ae} of
the second one.

For instance, if $\Tnat$ is the \theory\ of Presburger arithmetic and
$\Tarrays$ is the \theory\ of arrays, then:
\begin{itemize}
\item{$0 \leq x \leq n$ is a formula of $\Tnat$ ($x$ denotes a variable and $n$ denotes a constant symbol of sort $\snat$).}
\item{$\select(t,x) \iseq a$ is a formula of $\Tarrays$ (stating that $t$ is a constant array).}
\item{$0 \leq x \leq n \Rightarrow \select(t,x) \iseq a$ (stating that $t$ is a constant on the interval $[0,n]$) is a formula obtained by combining
$\Tnat$ and $\Tarrays$ hierarchically.}
\end{itemize}

Such a combination cannot be viewed as a union of  disjoint \theories, since the  axioms contain function symbols from both \theories.
In this example, $\Tnat$ is a {\em base \theory} and $\Tarrays$ is a {\em \target\ \theory}.

\newcommand{\Xbase}{\X_{\bind}}
\newcommand{\Xnonbase}{\X_{\nbind}}
\newcommand{\funbase}{\F_{\bind}}
\newcommand{\funnonbase}{\F_{\nbind}}

\newcommand{\grounding}{$\base$-ground instance}
\newcommand{\gr}[1]{#1_{\base\downarrow}}
\newcommand{\groundings}{$\base$-ground instances}

More formally, we assume that the set of sorts $\allSorts$ is divided
into two disjoint sets $\base$ and $\nonbase$ such that for every
function $f: \asort_1\times\ldots\times\asort_n \rightarrow \asort$,
if $\asort \in \base$, then $\asort_1,\ldots,\asort_n \in \base$. A
term is a \emph{\bterm}\ if it is of a sort $\asort \in \base$ and a
\emph{\tterm}\ if it is of a
sort $\asort \in \nonbase$ and contains no non-variable \bterm.  In
the sequel we let $\Xbase \isdef \bigcup_{\asort \in \base} \X_{\asort}$
(resp. $\Xnonbase \isdef \bigcup_{\asort \in \nonbase} \X_{\asort}$) be the
set of base variables (resp. \target\ variables) and let $\funbase$
(resp. $\funnonbase$) be the set of function symbols whose co-domain
is in $\base$ (resp. $\nonbase$).
An \emph{\grounding\ of an expression $\expr$} is an expression
of the form $\expr\sigma$ where $\sigma$ is a
ground substitution of domain $\var(\expr) \cap \Xbase$.
Intuitively, an \grounding\ of $\expr$ is obtained from $\expr$
by replacing every variable of a sort $\asort \in \base$
(and only these variables) by a ground term of the same sort.

  \begin{definition}
  \label{def:clbase}
$\clbase$ denotes the set of \gclauses\ $C$ such that every term occurring in $C$ is a \bterm.
$\clnonbase$ denotes the set of \gclauses\ $C$ such that:
 \begin{enumerate}
 \item{Every non-variable term occurring in $C$ is a \tterm.}
 \item{For every atom $t \iseq s$ occurring in $C$, $t$ and $s$ are {\tterm}s.}
 \end{enumerate}
 \end{definition}

Notice that it follows from the definition that $\clbase \cap
\clnonbase = \emptyset$, since $\base$ and $\nonbase$ are disjoint.

\begin{definition}
\label{def:basetheory}
A \theory\ $(\I,\C)$ is a \emph{\btheory}\ if $\axof{\I} \subseteq \clbase$ and for every $S \in \C$, $S \subseteq \clbase$.
It is a \emph{\ttheory}\ if $\axof{\I} \subseteq \clnonbase$ and for every $S \in \C$, $S \subseteq \clnonbase$.
\end{definition}

\newcommand{\projec}[2]{{#1}^{#2}}
\newcommand{\projecbase}[1]{#1^{\bind}}
\newcommand{\projecnonbase}[1]{#1^{\nbind}}

\newcommand{\nestedcombin}{hierarchic expansion}
\newcommand{\nestedcombins}{hierarchic expansions}
\newcommand{\nestedcombinof}[2]{hierarchic expansion of #2 over #1}

\newcommand{\aparametersort}{parametric}
\newcommand{\anonparametersort}{non parametric}
\newcommand{\parametersorts}{parametric sorts}
\newcommand{\ps}[1]{{\tt PS}^{#1}}
\newcommand{\entailsrigid}[1]{\models^{r}}

\commentthis{
\begin{proposition}
Let $\T = (\I,\C)$ be a \theory. Let $I,J$ be two interpretations.
Assume that $I$ and $J$ coincide on every sort symbols and on every symbols
For every \Imapping{\ps{\T}} $\ff$, and for every
$I \in \I$ we have $\ff(I) \in \I$.
\end{proposition}

\begin{proof}
This is a direct consequence of the definition since
$\ff$ may be written on the form $\ff_1\circ \ldots \circ \ff_n$ where $\ps{\T} = \{ \asort_1,\ldots,\asort_n\}$
and $\ff_i$ is a \Imapping{\{ \asort_i \}}.
\end{proof}
}

Throughout this section, $\basetheory = (\I_\bind,\C_\bind)$ will
denote a \btheory\ and $\nonbasetheory = (\I_\nbind,\C_\nbind)$
denotes a \ttheory.  Base and \target\ \theories\ are combined as
follows:
\newcommand{\bpart}{base part}
\newcommand{\tpart}{\target\ part}


\begin{definition}
\label{def:hcombin}
The \emph{\nestedcombinof{$\basetheory$}{$\nonbasetheory$}} is the
\theory\ $\hcombin{\basetheory}{\nonbasetheory} = (\I,\C)$ defined as
follows:
\begin{enumerate}
\item{$\axof{\I} \isdef \axof{\I_\bind} \cup \axof{\I_\nbind}$.
         \label{cond:hcI}}
     \item{Every clause set in $\C$ is of the form $\{ C_i^\bind \vee
         C_i^\nbind\mid i \in [1..n] \}$, where $\{ C_i^\bind \mid i
         \in [1..n] \} \in \C_\bind$ and $\{ C_i^\nbind \mid i \in
         [1..n] \} \in \C_\nbind$.
 \label{cond:hcC}}
    \end{enumerate}

If $C$ is a clause in $\C$, then $\projecbase{C}$ is the
\emph{\bpart}\ of the clause and $\projecnonbase{C}$ is its
\emph{\tpart}.
If $S$ is a set of clauses in $\C$, then $\projecbase{S}$ and
$\projecnonbase{S}$ respectively denote the sets $\{ \projecbase{C}
\mid C \in S \}$ and $\{ \projecnonbase{C} \mid C \in S \}$, and are respectively
called the \emph{\bpart} and \emph{\tpart} of $S$.

\end{definition}

The following proposition shows that the decomposition in Condition
\ref{cond:hcC} is unique.

\begin{proposition}
For every clause $C$ occurring in a clause set in $\C$,
there exist two unique clauses $\projecbase{C}$ and $\projecnonbase{C}$ such that
$C = \projecbase{C} \vee \projecnonbase{C}$.
\end{proposition}

\begin{proof}
  The existence of two clauses $\projecbase{C}$, $\projecnonbase{C}$
  is a direct consequence of Condition \ref{cond:hcC} in Definition
  \ref{def:hcombin}. Uniqueness follows straightforwardly from Definition
  \ref{def:basetheory}.
\end{proof}

\begin{example}
\label{ex:exnest}
Consider the following clauses:

{\small
\[
\begin{tabular}{lll}
$c_1$ & $\{ x \not \geq a \vee \select(t,x) \iseq 1 \}$ & ($t$ is constant on $[a,\infty[$) \\
$c_2$ & $\{ x \not \geq a \vee x \not \leq b \vee \select(t,x) \iseq \select(t',x) \}$ & ($t$ and $t'$ coincide on $[a,b]$) \\
$c_3$ & $\{ \select(t,i) \iseq \select(t',i+1) \}$ & ($t$ and $t'$
coincide up to a shift) \\
$c_4$ & $\{ x \not \leq y \vee \select(t,x) \leq \select(t,y) \}$ & ($t$ is sorted) \\
$c_5$ & $\{ \select(t,x) \leq x \}$ & ($t$ is lower than the identity) \\
\end{tabular}
\]
}

Clauses $c_1$ and $c_2$ occur in $\hcombin{\Tnat}{\Tarrays}$, and for
instance, $\projecnonbase{c_1} = (\select(t,x) \iseq 1)$ and
$\projecbase{c_1} = (x \not \geq a)$. Clause  $c_3$ does not occur in
$\hcombin{\Tnat}{\Tarrays}$ because the atom $\select(t',i+1)$ of the
\target\ \theory\ contains the non-variable term $i+1$ of the base
\theory. However, $c_3$ can be equivalently written as follows:
\[
\begin{tabular}{lll}
$c_3'$ & $\{ j \not\iseq i+1 \vee \select(t,i) \iseq \select(t',j) \}$
&
\end{tabular}
\]
and $c_3'$ is in $\hcombin{\Tnat}{\Tarrays}$\footnote{However as we
  shall see in Section \ref{sect:appl}, our method cannot handle such
  axioms, except in some very particular cases. In fact, adding axioms
  relating two consecutive elements of an array easily yields
  undecidable \theories\ (as shown in \cite{Bradleybook}).}.  Clause
$c_4$ does not occur in $\hcombin{\Tnat}{\Tarrays}$, because
$\select(t,x) \leq \select(t',x)$ contains symbols from both $\Tnat$
(namely $\leq$) and $\Tarrays$ ($\select$) which contradicts Condition
\ref{cond:hcC} of Definition \ref{def:hcombin}.  However, $c_4$ can be
handled in this setting by considering a \emph{copy} $\Tnat'$ of
$\Tnat$ (with disjoint sorts and function symbols).  In this case,
$c_4$ belongs to $\hcombin{\Tnat}{(\Tarrays \cup \Tnat')}$, where
$\Tarrays \cup \Tnat'$ denotes the union of the \theories\ $\Tarrays$
and $\Tnat'$.  Of course $\Tnat'$ can be replaced by any other
\theory\ containing an ordering predicate symbol.  The same
transformation \emph{cannot} be used on the clause $c_5$, since
(because of the literal $\select(t,x) \leq x$) the sort of the indices
cannot be separated from that of the elements. Again, this is not
surprising because, as shown in \cite{Bradleybook}, such axioms (in
which index variables occur out of the scope of a $\select$)
easily make the theory undecidable.
\end{example}

\newcommand{\adequate}{adequate}
\newcommand{\deriv}[1]{{#1}_{\vee}^{\star}}
\newcommand{\derivinf}[1]{{#1}_{\vee}^{\omega}}
\newcommand{\derivt}[2]{{#1}_{\vee}^{#2}}

\newcommand{\diam}{\bullet}
\newcommand{\speccst}{\diam}
\newcommand{\botz}{\diam}
\newcommand{\litof}[1]{{\cal L}_\speccst(#1)}
\newcommand{\adequateone}[1]{$#1$-embeddable}
\newcommand{\sadequateone}{base-complete}
\newcommand{\adequatetwo}{\target-complete}
\newcommand{\Sadequateone}{Base-Complete}

\newcommand{\ste}[1]{\bs{#1}}
\newcommand{\instB}[2]{{#1}_{\downarrow#2}}

Since $\base$ and $\nonbase$ are disjoint, the boolean sort cannot occur both in $\base$ and $\nonbase$.
However, this problem can easily be overcome by considering two copies of this sort ($\bool$ and $\bool'$).

\subsection{Nested Combination of Instantiation Schemes}

\label{sect:conditions}


\newcommand{\bs}[1]{G_{#1}}
\newcommand{\abs}{G}
\newcommand{\BB}{\abs}

The goal of this section is to investigate how instantiation schemes
for $\basetheory$ and $\nonbasetheory$ can be combined in order to obtain an
instantiation scheme for $\hcombin{\basetheory}{\nonbasetheory}$. For instance, given two instantiation schemes for integers and arrays respectively, we want to \emph{automatically} derive an instantiation scheme handling mixed axioms such as those in Example \ref{ex:exnest}.
We
begin by imposing conditions on the  schemes under consideration.

\newcommand{\hierarchy}{hierarchy}


\newcommand{\Bt}[1]{G(#1)}

\newcommand{\uniform}{uniform}
\newcommand{\stuniform}{strongly uniform}


\subsubsection{Conditions on the \Target\ \Theory}

First, we investigate what conditions can be imposed on the
instantiation procedure for the \target\ \theory\ $\nonbasetheory$. What is needed is
\emph{not}  an instantiation procedure that is complete for
$\nonbasetheory$; indeed, since by definition every term of a sort in
$\base$ occurring in $\C_\nbind$ is a variable, such an instantiation
would normally replace every such variable by an arbitrary ground term
(a constant, for example). This is not satisfactory because in the
current setting, the value of these variables can be constrained by the
base part of the clause. This is why we shall assume that the considered
procedure is complete for every clause set that is obtained from
clauses in $\C_\nbind$ by grounding the variables in $\Xbase$,
\emph{no matter the grounding instantiation}.


\newcommand{\bmapping}{$\base$-mapping}

\begin{definition}
  An \emph{\bmapping}\ is a function $\am$ from $\baseterm$ to
  $\baseterm$.  Such a mapping is extended straightforwardly into a
  function from expressions to expressions: for every expression
  (term, atom, literal, clause or set of clauses) $\expr$,
  $\am(\expr)$ denotes the expression obtained from $\expr$ by
  replacing every term $t \in \baseterm$ occurring in $\expr$ by
  $\am(t)$.

  An instantiation procedure $\Theta$ is \emph{\preserving}\ iff for
  every \bmapping\ $\am$, and every clause $C$ in a set $S$, $C \in
  \Theta(S) \Rightarrow \am(C) \in \Theta(\am(S))$.
\end{definition}

We may now define \emph{\adequatetwo} instantiation
procedures. Intuitively, such a procedure must be complete on those
sets in which the only terms of a sort in $\base$ that occur are
ground, the instances cannot depend on the names of the terms in
$\baseterm$ and the addition of
 information cannot make the procedure  less instantiate a clause set.


\begin{definition}
\label{def:theta2}
An instantiation procedure $\Theta$ is \emph{\adequatetwo}\ if the following conditions hold:

\begin{enumerate}
\item{\label{theta2:comp}For all sets $S \in \C_{\nbind}$ and all sets
    $S'$ such that every clause in $S'$ is an \grounding\ of a clause
    in $S$, $S'$ and $\Theta(S')$ are $\T$-equisatisfiable. }
 \item{\label{theta2:pres}$\Theta$ is \preserving.
}
\item{$\Theta$ is monotonic: $S' \subseteq S \Rightarrow \Theta(S') \subseteq \Theta(S)$. \label{theta2:mono}}
\end{enumerate}
\end{definition}


\subsubsection{Conditions on the Base \Theory}

Second, we impose conditions on the instantiation procedure for the base \theory\ $\basetheory$.
We need the following definitions:

\begin{definition}
\label{def:inst}
Let $S$ be a set of clauses and let $\abs$ be a set of terms.
We denote by $\instB{S}{\abs}$ the set of clauses of the form
$C\sigma$, where
$C \in S$ and $\sigma$ maps every variable in $C$ to a term of the same sort in $\abs$.
\end{definition}

\begin{proposition}
\label{prop:instmono}
Let $S$ be a set of clauses and let $\abs$ and $\abs'$ be two sets of ground terms. If $\abs \subseteq \abs'$ then $\instB{S}{\abs} \subseteq \instB{S}{\abs'}$.
\end{proposition}



\begin{definition}
\label{def:deriv}
If $S$ is a set of clauses, we denote by $\deriv{S}$ the set of
clauses of the form $\bigvee_{i=1,\ldots, n} C_i\sigma_i$
such that for every $i \in [1,n]$, $C_i \in S$ and $\sigma_i$ is a
pure substitution.
\end{definition}

\begin{example}
Let $S = \{ p(x,y) \}$. Then $\deriv{S}$ contains among others the clauses
$p(x,x), p(x,y)$, $p(x,y) \vee p(z,u)$, $p(x,y) \vee p(y,x)$, $p(x,y) \vee p(y,z) \vee p(z,u)$, etc.
\end{example}
\begin{definition}
\label{def:theta1s}
An instantiation procedure
$\Theta$ for $\basetheory$ is \emph{\sadequateone}\ iff the following conditions hold:
\begin{enumerate}
 \item{For every $S \in \C_\bind$ there exists a \emph{finite} set of terms $\bs{S}$ such that
$\Theta(S) = \instB{S}{\bs{S}}$
and $\Theta(S)$ and $S$ are
$\basetheory$-equisatisfiable.
\label{theta1s:bs}}

\item{If $S \subseteq S'$ then $\bs{S} \subseteq \bs{S'}$. \label{theta1s:mono}}
\item{For every clause set $S \in \C$, $\bs{\deriv{S}}\subseteq
    \bs{S}$.
\label{theta1s:disj}}
\end{enumerate}

\end{definition}

Obviously these conditions are much stronger than those of Definition \ref{def:theta2}. Informally, Definition \ref{def:theta1s} states that:
\begin{enumerate}
\item{All variables must be instantiated in a uniform\footnote{Of course sort constraints must be taken into account.} way by ground
    terms, and satisfiability must be preserved.}
\item{The instantiation procedure is monotonic.}
\item{The considered set of ground terms does not change when new
clauses are added to $S$, provided that these clauses are obtained from clauses already occurring in $S$ by disjunction and pure instantiation only.}
\end{enumerate}

\newcommand{\sutop}{\gamma^{\speccst}}
\newcommand{\insttop}[1]{#1\sutop}
\newcommand{\instcc}[2]{#1[#2/\speccst]}

\subsubsection{Definition of the Combined Instantiation Scheme}

We now define an instantiation procedure for
$\hcombin{\basetheory}{\nonbasetheory}$.  Intuitively this procedure
is defined as follows.
\begin{enumerate}
\item{First,  the \tpart\ of each clause in $S$ is extracted and  all
    base variables are instantiated by arbitrary constant symbols
    $\diam$ (one for each base sort). }
\item{The instantiation procedure for $\nonbasetheory$ is applied on
    the resulting clause set. This instantiates all {\tvariable}s\
    (but not the {\bvariable}s, since they have already been
    instantiated at Step $1$). }
\item{All the substitutions on {\target} variables from Step 2 are
    applied to the \emph{initial} set of clauses.}
\item{Assuming the instantiation procedure for $\basetheory$ is
    \sadequateone, if this procedure was applied to the \bpart\ of the
    clauses, then by Condition \ref{theta1s:bs} of Definition
    \ref{def:theta1s}, the {\bvariable}s in the {\bpart} of the
    clauses would be uniformly instantiated by some set of terms
    $\abs$.  All base variables and all occurrences of constants $\diam$
    are replaced by all possible terms in $\abs$.}

\end{enumerate}

\begin{example}
  Assume that $\basetheory = \Tnat$, $\nonbasetheory = \Tfol$ and that
  $\F$ contains the following symbols: $a: \snat$, $b: \snat$, $c:
  \asort$ and $p: \snat \times \asort \rightarrow \sbool$.  Consider
  the set $S = \{ x \not \leq a \vee p(x,y), u \not \leq b \vee \neg
  p(u,c) \}$.

\begin{enumerate}
\item{We compute the set $\projecnonbase{S} = \{ p(x,y), \neg p(u,c)
    \}$ and replace every base variable by $\diam$. This yields the
    set: $ \{ p(\speccst,y), \neg p(\speccst,c) \}$.}
\item{We apply an instantiation procedure for $\Tfol$\footnote{There
      exist several instantiation procedures for $\Tfol$, one such
      example is given in Section \ref{sect:fol}.}. Obviously,
    this procedure should instantiate the variable $y$ by $c$,
    yielding $ \{ p(\speccst,c), \neg p(\speccst,c) \}$.}
\item{We apply the (unique in our case) substitution $y \mapsto c$ to
    the initial clauses: $\{ x \not \leq a \vee p(x,c), u \not \leq b
    \vee \neg p(u,c) \}$. Note that at this point all the remaining
    variables are in $\Xbase$.}
\item{We compute the set of clauses $\projecbase{S} = \{ x \not \leq
    a, u \not \leq b \}$ and the set of terms $\bs{\projecbase{S}}$.
    It should be intuitively clear\footnote{A formal definition of an
      instantiation procedure for this fragment of Presburger arithmetic
      will be given in Section \ref{sect:pa}.} that $x$ must be
    instantiated by $a$ and $u$ by $b$, yielding $\bs{\projecbase{S}}
    = \{ a,b \}$.}
\item{We thus replace all base variables by every term in $\{ a, b \}$
    yielding the set $\{ a \not \leq a \vee p(a,c), b \not \leq a \vee
    p(b,c), a \not \leq b \vee \neg p(a,c), b \not \leq b \vee \neg
    p(b,c)\}$, i.e., after simplification, $\{ p(a,c), b \not \leq a
    \vee p(b,c), a \not \leq b \vee \neg p(a,c), \neg p(b,c)\}$. It is
    straightforward to check that this set of clauses is
    unsatisfiable. Any SMT-solver capable of handling arithmetic and
    propositional logic can be employed to test the satisfiability of
    this set.}
\end{enumerate}
\end{example}

The formal definition of the procedure is given below.
Let $\sutop$ be a substitution mapping every variable of a sort $\asort \in \base$
to an arbitrary constant symbol $\speccst_\asort$ of sort $\asort$.

\begin{definition}
\label{def:hcomb}
Let $\Theta_\bind$ be a \sadequateone\ instantiation procedure
and  $\Theta_\nbind$ be a \adequatetwo{} instantiation
procedure.
$\hcombinproc{\Theta_\bind}{\Theta_\nbind}(S)$ is defined
as the set of clauses of the form
$(\projecbase{C} \vee \projecnonbase{C})\theta'\sigma$ where:
\begin{itemize}
\item{$C \in S$.}
\item{$\projecnonbase{C}\sutop\theta \in \Theta_\nbind(\insttop{\projecnonbase{S}})$.}
\item{$\theta'$ is obtained from $\theta$ by replacing every occurrence of a constant symbol $\speccst_\asort$ in the co-domain of $\theta$ by a fresh variable of the same sort.}
\item{$\sigma$ maps every variable in $C\theta'$ to a term of the same sort in
$\ste{\projecbase{S}}$.}
\end{itemize}
\end{definition}

The following proposition is straightforward to prove and states the
soundness of this procedure:

\begin{proposition}
  Let $\Theta_\bind$ be a \sadequateone\ instantiation procedure and
  let $\Theta_\nbind$ be a \adequatetwo{} instantiation procedure.
  For every set of clauses $S \in \C$,
  $\hcombinproc{\Theta_\bind}{\Theta_\nbind}(S)$ is a set of ground
  instances of clauses in $S$. Thus if
  $\hcombinproc{\Theta_\bind}{\Theta_\nbind}(S)$ is
  $\hcombin{\basetheory}{\nonbasetheory}$-unsatisfiable, then so is
  $S$.
\end{proposition}

Several examples of
concrete instantiation procedures satisfying the conditions of
Definitions \ref{def:theta2} and \ref{def:theta1s} are provided in Section \ref{sect:appl}.

\subsection{Completeness}

\label{sect:complete}

The remainder of this section is devoted to the proof of the main
result of this paper, namely that the procedure
$\hcombinproc{\Theta_1}{\Theta_2}$ is complete for
$\hcombin{\basetheory}{\nonbasetheory}$:

\begin{theorem}
\label{theo:comp}
Let $\Theta_\bind$ be a \sadequateone\ instantiation procedure (for $\basetheory$)
and let $\Theta_\nbind$ be a \adequatetwo{} instantiation
procedure (for $\nonbasetheory$).
Then $\hcombinproc{\Theta_\bind}{\Theta_\nbind}$ is \complete\ for
$\hcombin{\basetheory}{\nonbasetheory}$; furthermore,
this procedure is
monotonic 
and \preserving.
\end{theorem}

The rest of the section (up to Page \pageref{sect:appl}) can be skipped entirely by readers not
interested in the more theoretical aspects of the
work. The proof of this theorem relies on a few intermediate
results that are developed in what follows.

\newcommand{\depon}{\triangleleft}


\subsubsection{Substitution Decomposition}

\begin{definition}
A substitution $\sigma$ is a \emph{\bsubstitution}\ iff
$\dom(\sigma) \subseteq \Xbase$. It is a \emph{\tsubstitution}\ iff
$\dom(\sigma) \subseteq \Xnonbase$ and for every $x \in \dom(\sigma)$,
$x\sigma$ contains no non-variable \bterm.
\end{definition}

We  show that every ground substitution can be decomposed into two parts: a \tsubstitution\ and a
\bsubstitution. We begin by an example:

\begin{example}
  Assume that $\basetheory = \Tnat$, $\nonbasetheory = \Tfol$ and that
  $\F$ contains the following symbols: $f: \asort \times \snat
  \rightarrow \asort, c: \asort $.  Consider the ground substitution
  $\sigma = \{ x \mapsto f(c,s(0)), y \mapsto f(f(c,0),0), n \mapsto
  s(0) \}$. We can extract from $\sigma$ a {\tsubstitution} by
  replacing all subterm-maximal {\bterm}s by variables, thus obtaining
  $\sigma_{\nbind} = \{ x \mapsto f(c,n), y \mapsto f(f(c,m),m) \}$,
  and then construct the {\bsubstitution} $\sigma_{\bind} = \{ n
  \mapsto s(0), m \mapsto 0 \}$ such that $\sigma =
  \sigma_{\nbind}\sigma_{\bind}$. Note that $\sigma_{\nbind}$ is not
  ground and that $\dom(\sigma_{\bind}) \not \subseteq \dom(\sigma)$.
\end{example}

The following result generali{\z}es this construction:

\begin{proposition}
\label{prop:decsubs}
Every ground substitution $\sigma$ can be decomposed into a product
$\sigma = (\sigma_\nbind\sigma_\bind)|_{\dom(\sigma)}$ where
$\sigma_\nbind$ is a \tsubstitution, $\sigma_{\bind}$ is a
\bsubstitution, and for all $x \in \dom(\sigma_\bind) \setminus
\dom(\sigma)$,
\begin{itemize}
\item{
$\forall y\in \dom(\sigma_\bind) \cap \dom(\sigma), x\sigma_\bind \not = y\sigma_\bind$,}
 \item{$\forall y \in \dom(\sigma_\bind) \setminus \dom(\sigma), y\sigma = x\sigma \Rightarrow x=y$.}
 \end{itemize}
\end{proposition}

\newcommand{\yt}[1]{\nu(#1)}
\newcommand{\ytfun}{\nu}

\begin{proof}
  Let $E$ be the set of subterm-maximal {\bterm}s occurring in terms
  of the form $x\sigma$, with $x \in \dom(\sigma)$.  Let $\ytfun$ be a
  (partial) function mapping every term $t \in E \cap \Ran{\sigma}$ to
  an arbitrarily chosen variable $\yt{t}$ such that $\yt{t}\sigma =
  t$.  This function $\ytfun$ is extended into a total function on $E$
  by mapping all terms $t$ for which $\yt{t}$ is undefined to pairwise
  distinct new variables, not occurring in $\dom(\sigma)$. Note that
  $\ytfun$ is injective by construction.  The substitutions
  $\sigma_\bind$ and $\sigma_\nbind$ are defined as follows:
\begin{itemize}
\item{$\dom(\sigma_\nbind) \isdef \dom(\sigma) \cap \Xnonbase$ and
    $x\sigma_\nbind$ is the term obtained by replacing every
    occurrence of a term $t \in E$ in $x\sigma$ by $\yt{t}$;}
\item{$\dom(\sigma_\bind) \isdef [\dom(\sigma) \cap \Xbase] \cup \yt{E}$;
    if $x = \yt{t}$ for some term $t\in E$, then $x\sigma_\bind \isdef t$;
    otherwise, $x\sigma_\bind \isdef x\sigma$. Note that
    $\sigma_\bind$ is well-defined, since by definition if $\yt{t} =
    \yt{s}$ then $t=s$.}
\end{itemize}
By construction, $\sigma_\nbind$ is a {\tsubstitution} and
$\sigma_\bind$ is a {\bsubstitution}. Furthermore, since
$\yt{t}\sigma_{\bind} = t$, $x\sigma_{\nbind}\sigma_\bind = x\sigma$
for every $x \in \dom(\sigma) \cap \Xnonbase$. Similarly, for every $x
\in \dom(\sigma) \cap \Xbase$, $x\sigma_{\nbind}\sigma_\bind =
x\sigma_\bind = x\sigma$ and therefore $\sigma =
(\sigma_\nbind\sigma_\bind)|_{\dom(\sigma)}$.  Let $x \in
\dom(\sigma_\bind) \setminus \dom(\sigma)$. By definition of
$\sigma_\bind$, $x$ is of the form $\yt{t}$ for some $t \in E$, and
there is no variable $y \in \dom(\sigma)$ such that $y\sigma = t$,
since otherwise $\yt{t}$ would have been defined as $y$. Thus $\forall
y\in \dom(\sigma_\bind) \cap \dom(\sigma), x\sigma_\bind \not =
y\sigma = y\sigma_\bind$. Now if $y \in \dom(\sigma_\bind) \setminus
\dom(\sigma)$ and $x\sigma_\bind = y\sigma_\bind$, then $y$ is also of
the form $\yt{s}$ for some $s \in E$ and we have $x\sigma_\bind = t$ and
$y\sigma_\bind = s$, hence $t = s$ and $x = y$.
\end{proof}

\newcommand{\nonbasepart}[2]{#1|_{#2}}

\subsubsection{Partial Evaluations}

Given a set of clauses $S$ in $\hcombin{\basetheory}{\nonbasetheory}$
and an interpretation $I$ of $\basetheory$, we consider a set of
clauses $S'$ of $\nonbasetheory$ by selecting those ground instances
of clauses in $S$ whose {\bpart} evaluates to {\false} in $I$ and
adding their {\tpart} to $S'$. More formally:
\begin{definition}
For every clause $C \in \C_\bind$ and for every interpretation $I \in \I_\bind$, we denote by
$\falsifyingsubs{I}{C}$ the set of ground substitutions $\eta$ of domain $\var(C)$ such that
$I \not \models C\eta$.
Then, for every $S \in \C$ we define:
$$\nonbasepart{S}{I} \isdef \{ \projecnonbase{C}\eta \mid C \in S, \eta \in \falsifyingsubs{I}{\projecbase{C}} \}.$$
\end{definition}

\begin{example}
Let $S = \{ x \not \iseq a \vee P(x),\ y < 2 \vee Q(y,z) \}$ be a set of clauses in $\hcombin{\Tnat}{\Tfol}$, where $x,y,a$ are of sort $\snat$ and $z$ is a variable of a sort distinct from $\snat$.
Let $I$ be the interpretation of natural numbers such that $\interp{a}{I} = 1$. Then $\falsifyingsubs{I}{x \not \iseq a} = \{ x \mapsto 1 \}$ and  $\falsifyingsubs{I}{y < 2} = \{ y \mapsto k \mid k \in \mathbb{N}, k \geq 2 \}$. Therefore $\nonbasepart{S}{I} = \{ P(1) \} \cup \{ Q(k,z) \mid k \in \mathbb{N}, k \geq 2 \}$.
\end{example}



The following lemma shows that $\nonbasepart{S}{I}$ is
$\nonbasetheory$-unsatisfiable when $S$ is
$\hcombin{\basetheory}{\nonbasetheory}$-unsatisfiable.

\begin{lemma}
\label{lem:ri}
For every $\hcombin{\basetheory}{\nonbasetheory}$-unsatisfiable set of clauses $S \in \C$ and for every $I \in \I_\bind$, $\nonbasepart{S}{I}$ is $\nonbasetheory$-unsatisfiable.
\end{lemma}

\begin{proof}
Let $\hcombin{\basetheory}{\nonbasetheory} =(\I,\C)$.
  Assume that $\nonbasepart{S}{I}$ is $\nonbasetheory$-satisfiable,
  i.e. that there exists an interpretation $J \in \I_\nbind$
  validating $\nonbasepart{S}{I}$. W.l.o.g. we assume that the domain
  of $J$ is disjoint from that of $I$.  We  construct an
  interpretation $K \in \I$ satisfying $S$, which will yield a
  contradiction since $S$ is $\hcombin{\basetheory}{\nonbasetheory}$-unsatisfiable by hypothesis.

  For all sort symbols $\asort \in \base$ and for all $e \in
  \domofsort{\asort}{I}$, we denote by $\nf{e}$ an arbitrarily chosen
  ground term in $\baseterm$ such that $\eqcl{\nf{e}}{I} = e$\footnote{$\nf{e}$ always exists since we restricted ourselves to interpretations such that, for every $\asort \in \allSorts$,
$\domofsort{\asort}{I} = \{ \valueof{t}{I} \mid t \in \gt{\asort} \}$.}.  If
  $\expr$ is a ground expression, we denote by $\nff{\expr}$ the
  expression obtained from $\expr$ by replacing every term $t$ by
  $\nf{\eqcl{t}{I}}$; by construction $\valueof{\expr}{I} =
  \valueof{\nff{\expr}}{I}$.  Let $\psi: \domof{I} \uplus \domof{J}
  \rightarrow \domof{J}$ be the function defined for every element
  $e \in \domof{I} \cup \domof{J}$ as follows:
\begin{itemize}
\item{if $e \in \domofsort{\asort}{I}$ then $\psi(e) \isdef  \valueof{\nf{e}}{J}$;}
\item{otherwise $\psi(e) \isdef e$.}
\end{itemize}
We define the interpretation $K$ by combining $I$ and $J$ as
follows:
\begin{itemize}
\item{$K$ coincides with $I$ on $\base$ and on every function symbol whose co-domain is in $\base$.}
    \item{$K$ coincides with $J$ on $\nonbase$.}
    \item{For all function symbols $f\in \funnonbase$ of arity $n$,
        $\interp{f}{K}(e_1,\ldots,e_n) \isdef
        \interp{f}{J}(\psi(e_1),\ldots,\psi(e_n))$. Note that
        $\interp{f}{K}$ is well-defined since by definition of $\psi$,
        if $e \in \domofsort{\asort}{K}$ then  $\psi(e) \in
        \domofsort{\asort}{J}$.}
\end{itemize}

Let $\expr$ be a ground expression (term, atom, literal, clause or
\gclause) such that $\nff{\expr} = \expr$. Assume that $\expr$ is a
ground instance of an expression occurring in a clause in
$\clnonbase$. We prove by structural induction on $\expr$ that
$\valueof{\expr}{J} = \psi(\valueof{\expr}{K})$.

\begin{itemize}
\item{If $\expr$ is a term of a sort in $\base$ then since $I$ and $K$
    coincide on $\base \cup \funbase$, we have $\valueof{\expr}{K} =
    \valueof{\expr}{I}$.  By hypothesis $\nff{\expr} = \expr$, thus
    $\nf{\valueof{\expr}{I}} = \expr$ and by definition of $\psi$,
    $\psi(\valueof{\expr}{K}) = \psi(\valueof{\expr}{I}) =
    \valueof{\nf{\expr}}{J} = \valueof{\expr}{J}$.  }
\item{If $\expr$ is of the form $f(t_1,\ldots,t_n)$ where $f \in
    \funnonbase$, then by definition $\valueof{\expr}{J} =
    \interp{f}{J}(\valueof{t_1}{J},\ldots,\valueof{t_n}{J})$ and by
    the induction hypothesis, $\valueof{t_i}{J} =
    \psi(\valueof{t_i}{K})$ for $i \in [1,n]$.  Again by
    definition, $\valueof{\expr}{K} =
    \interp{f}{J}(\psi(\valueof{t_1}{K}),\ldots,\psi(\valueof{t_n}{K}))
    = \interp{f}{J}(\valueof{t_1}{J},\ldots,\valueof{t_n}{J}) =
    \valueof{\expr}{J}$. Thus, since the domains of $I$ and $J$ are
    disjoint,  $\valueof{\expr}{J} \not \in
    \domofsort{\base}{I}$, hence $\psi(\valueof{\expr}{J}) =
    \valueof{\expr}{J}$.}
\item{If $\expr$ is an atom of the form $t_1 \iseq t_2$ then $t_1,t_2
    \not \in \base$. Indeed $\expr$ occurs in a ground instance of a
    clause $C$ occurring in $\clnonbase$ and by Definition
    \ref{def:clbase}, such clauses cannot contain equalities between
    {\bterm}s. Thus we have $\psi(\valueof{t_i}{K}) =
    \valueof{t_i}{K}$ (for $i =1,2$) and the proof is straightforward.}
\item{The proof is immediate if $\expr$ is a literal or a (possibly infinite) disjunction of literals.}
\end{itemize}

Since $J \models \nonbasepart{S}{I}$ and all specifications are
assumed to be {\gdefin} (see Definition \ref{def:gdefin}), we deduce
that $K \models \nonbasepart{S}{I} \cup \axof{\I_\nbind}$.  Indeed,
for the sake of contradiction, assume that there exists an \gclause\
$C \in \nonbasepart{S}{I} \cup \axof{\I_\nbind}$ and a ground
substitution $\theta$ of domain $\var(C)$ such that $K \not \models
C\theta$. Since $K \models t \iseq \nff{t}$ for every term $t$,
necessarily $K \not \models C\theta'$ where $x\theta' \isdef
\nff{x\theta}$.  But then $\nff{C\theta'} = C\theta'$ and since
$\valueof{\expr}{J} = \psi(\valueof{\expr}{K})$, we conclude that $J
\not \models C\theta'$ which is impossible since by hypothesis $J$ is
an $\nonbasetheory$-model of $\nonbasepart{S}{I}$.

We now prove that $K \models S$.  Let $C
\in S$ and $\eta$ be a ground substitution of domain
$\var(C)$. W.l.o.g. we assume that $\forall x \in \var(C), \nff{x\eta}
= x\eta$.  Let $\eta_\bind$ (resp. $\eta_\nbind$) be the restriction
of $\eta$ to the variables of a sort in $\base$ (resp. in $\nonbase$).
If $I \models \projecbase{C}\eta_\bind$ then $K \models
\projecbase{C}\eta_\bind$ because $K$ and $I$ coincide on $\base \cup
\funbase$, and consequently $K \models C\eta$ (since $C\eta \supseteq
\projecbase{C}\eta_\bind$).  If $I \not \models
\projecbase{C}\eta_\bind$ then $\eta_\bind \in \falsifyingsubs{I}{C}$,
hence $\projecnonbase{C}\eta_\bind \in \nonbasepart{S}{I}$.  Again $K
\models C\eta_\bind$ hence $K \models C\eta$; therefore  $K
\models S$.

Finally, since $K$ coincides with $I$ on $\base \cup \funbase$ we have
$K \models \axof{\I_\bind}$. This proves that $K$ is an
$\hcombin{\basetheory}{\nonbasetheory}$-model of $S$, which is
impossible.
\end{proof}

\subsubsection{Abstraction of Base Terms}

Lemma \ref{lem:ri} relates the
$\hcombin{\basetheory}{\nonbasetheory}$-unsatisfiability of a set of
clauses $S$ to the $\nonbasetheory$-unsatisfiability of sets of the
form $\nonbasepart{S}{I}$.  By definition, $\nonbasepart{S}{I}$ is of
the form $S'\sigma$, for some clause set $S' \in \C_\nbind$ and for
some ground \bsubstitution\ $\sigma$.  However, since neither
$\axof{\I_\nbind}$ nor $\C_{\nbind}$ contains symbols of a sort in
$\base$, the interpretation of the ground {\bterm}s of $S'$ in an
interpretation of $\I_{\nbind}$ is arbitrary: changing the values of
these terms does not affect the $\nonbasetheory$-satisfiability of
the formula.  Thus  the actual concrete values of the ground
{\bterm}s does not matter: what is important is only how these terms
compare to each other.

\begin{example}
  Assume that $\nonbasetheory = \Tfol$, $p: \snat \times \asort
  \rightarrow \sbool$, $a: \asort$, and let $S = \{ p(x,z), \neg
  p(y,a) \}$.  Consider $\sigma: \{ x \mapsto 0, y \mapsto 0 \}$, clearly,
  $S\sigma \models^{\nonbasetheory} \Box$.  But  also  $S\{ x
  \mapsto s(0), y \mapsto s(0) \} \models^{\nonbasetheory} \Box$ and
  more generally $S\{ x \mapsto t, y \mapsto t \}
  \models^{\nonbasetheory} \Box$.  On the other hand, $S\{ x \mapsto
  0, y \mapsto s(0) \} \not \models^{\nonbasetheory} \Box$ and more
  generally $S\{ x \mapsto t, y \mapsto t' \} \not
  \models^{\nonbasetheory} \Box$ if $t,t'$ are distinct integers.
\end{example}

Therefore, if $S\sigma \models^{\nonbasetheory} C\sigma$ for some
\bsubstitution\ $\sigma$ then actually $S\theta
\models^{\nonbasetheory} C\theta$, for every substitution $\theta$
such that $x\theta = y\theta \Leftrightarrow x\sigma=y\sigma$.  This
will be formali{\z}ed in the following definitions and lemma.  We first
introduce an unusual notion of semantic entailment. The intuition is
that variables in $\base$ are considered as ``rigid'' variables that
must be instantiated by \emph{arbitrary} ground terms:

\begin{definition}
Let $S \in \C_{\nbind}$.
We write $S \entailsrigid{\T} C$ iff for every ground substitution
of domain $\Xbase$,  $S\sigma \models^{\nonbasetheory} C\sigma$.
\end{definition}

\begin{example}
  Assume that $\nonbasetheory = \Tfol$.  Let $a: \asort$, $p: \snat
  \times \asort \rightarrow \sbool$ and $q: \snat \rightarrow \sbool$,
  where $\snat \in \base$, $\asort \in \nonbase$.  Let $S = \{ p(x,y),
  \neg p(u,a) \vee q(u) \}$, where $x,y,u$ are variables.  Then $S
  \entailsrigid{\Tfol} q(x)$, but $S \not \entailsrigid{\Tfol} q(0)$.
  Note that $x$ denotes the same variable in $S$ and $q(x)$ (the variables are \emph{not} renamed).
\end{example}

\newcommand{\purify}[1]{\langle#1\rangle}

\begin{definition}
  For every substitution $\sigma$ we denote by $\purify{\sigma}$ an
  arbitrarily chosen pure substitution such that $x\sigma = y\sigma
  \Rightarrow x\purify{\sigma} = y\purify{\sigma}$, for every $x,y \in
  \X$.
\end{definition}

Note that such a substitution always exists.  The next lemma can be
viewed as a generali{\z}ation lemma: it shows that the values of the
ground {\bterm}s can be abstracted into variables.

\begin{lemma}
\label{lem:gene}
Let $S \in \C_{\nbind}$ and $\sigma$ be a \bsubstitution\ such that
$\dom(\sigma) \subseteq \Xbase$.  If $S\sigma \models^{\nonbasetheory}
C\sigma$ then $S\purify{\sigma} \entailsrigid{\nonbasetheory}
C\purify{\sigma}$.
\end{lemma}

\begin{proof}
  Let $\theta$ be a substitution of domain $\Xbase$.  We assume that
  there exists an $I \in \I_\nbind$ such that $I \models
  S\purify{\sigma}\theta$ and $I \not \models C\purify{\sigma}\theta$,
  and we show that a contradiction can be derived.

  For every ground term $t$, we denote by $\fff(t)$ the ground term
  obtained from $t$ by replacing every ground subterm of the form
  $x\sigma$ by $x\purify{\sigma}\theta$.  $\fff$ is well-defined:
  indeed, if $x\sigma = y\sigma$, then by definition of
  $\purify{\sigma}$, $x\purify{\sigma} = y\purify{\sigma}$ thus
  $x\purify{\sigma}\theta = y\purify{\sigma}\theta$.  Let $J$ be the
  interpretation defined as follows\footnote{Intuitively, $J$ interprets every
  \bterm\ as itself and coincides with $I$ on {\tterm}s.}:
\begin{itemize}
\item{If $\asort \in \base$ then $\domofsort{\asort}{J} \isdef \gt{\asort}$.}
\item{If $f$ is a symbol of rank $\asort_1 \times\ldots\times\asort_n
    \rightarrow \asort$ where $\asort_1,\ldots,\asort_n,\asort \in
    \base$ then $\inter{f}{J}(t_1,\ldots,t_n) \isdef
    f(t_1,\ldots,t_n)$.}
\item{If $f$ is a symbol of rank $\asort_1 \times\ldots\times\asort_n
    \rightarrow \asort$ where $\asort \not \in \base$ then
    $\inter{f}{J}(t_1,\ldots,t_n) \isdef
    \inter{f}{I}(t_1',\ldots,t_n')$ where for every $i \in [1,n]$,
    $\asort_i \in \nonbase \Rightarrow t_i' = \valueof{t_i}{J}$ and
    $\asort_i \in \base \Rightarrow t_i' = \valueof{\fff(t_i)}{I}$.}
\end{itemize}
By construction, $\valueof{s}{J} = s$ for every ground \bterm\ $s$; we
prove that for every ground \tterm\ $t$, $\valueof{t}{J} =
\valueof{\fff(t)}{I}$, by induction on $t$. If $t =
f(t_1,\ldots,t_n)$, then $\valueof{t}{J} =
\inter{f}{I}(t_1',\ldots,t_n')$ where for every $i \in [1,n]$,
$\asort_i \in \nonbase \Rightarrow t_i' = \valueof{t_i}{J}$ and
$\asort_i \in \base \Rightarrow t_i' =\valueof{\fff(t_i)}{I}$.  By the
induction hypothesis, $\asort_i \in \nonbase \Rightarrow t_i' =
\valueof{\fff(t_i)}{I}$.  Thus $\valueof{t}{J} =
\inter{f}{I}(\valueof{\fff(t_1)}{I},\ldots, \valueof{\fff(t_n)}{I}) =
\valueof{\fff(t)}{I}$.

Now let $\sigma'$ be a ground substitution with a domain in
$\Xnonbase$, and let $\theta' \isdef \fff\circ\sigma'$.  We prove that for every expression $\expr$ occurring
in $S \cup \{ C \}$ that is not a \bterm,
$\valueof{\expr\sigma\sigma'}{J} =
\valueof{\expr\purify{\sigma}\theta\theta'}{I}$.

\begin{itemize}
\item{Assume that $\expr$ is a variable $x$ in $\Xnonbase$.  Then
    $\valueof{\expr\sigma\sigma'}{J} = \valueof{x\sigma'}{J}$, and by the
    previous relation we get $\valueof{\expr\sigma\sigma'}{J} =
    \valueof{\fff(x\sigma')}{I} = \valueof{x\theta'}{I} =
    \valueof{\expr\purify{\sigma}\theta\theta'}{I}$.}
\item{Assume that $\expr$ is a \target\ term of the form
    $f(t_1,\ldots,t_n)$.  Then by the result above,
    $\valueof{\expr\sigma\sigma'}{J} =
    \valueof{\fff(\expr\sigma\sigma')}{I}$.  By definition of $\fff$
    we have $\fff(\expr\sigma\sigma') =
    f(\fff(t_1\sigma\sigma'),\ldots,\fff(t_n\sigma\sigma'))$,
    therefore, $\valueof{\expr\sigma\sigma'}{J} =
    \inter{f}{I}(\valueof{\fff(t_1\sigma\sigma')}{I},\ldots,
    \valueof{\fff(t_n\sigma\sigma')}{I})$.  For $i \in [1,n]$, if
    $t_i$ is a {\tterm} then by the result above
    $\valueof{\fff(t_i\sigma\sigma')}{I} =
    \valueof{t_i\sigma\sigma'}{J}$ and by the induction hypothesis,
    $\valueof{\fff(t_i\sigma\sigma')}{I} =
    \valueof{t_i\purify{\sigma}\theta\theta'}{I}$. Otherwise, $t_i$ is
    a {\bterm}, and must necessarily be a variable, thus
    $\fff(t_i\sigma) = t_i\purify{\sigma}\theta$.  Therefore
    $\fff(t_i\sigma\sigma') = \fff(t_i\sigma) =
    t_i\purify{\sigma}\theta = t_i\purify{\sigma}\theta\theta'$.
    Therefore $\valueof{\expr\sigma\sigma'}{J} =
    \inter{f}{I}(\valueof{t_1\purify{\sigma}\theta\theta'}{I}, \ldots,
    \valueof{t_n\purify{\sigma}\theta\theta'}{I}) =
    \valueof{\expr\purify{\sigma}\theta\theta'}{I}$.  }
\item{The proof is similar if $\expr$ is of the form $t \iseq s$, $t \not \iseq s$ of $\bigvee_{i=1}^n l_i$.}
\end{itemize}

We thus conclude that for every clause $D \in S \cup \{ C \} \cup
\axof{\I}$, $J \models D\sigma\sigma'$ iff $I \models
D\purify{\sigma}\theta\theta'$.  Since $I \models
S\purify{\sigma}\theta\cup \axof{\I_\nbind}$, we deduce that $J
\models S\sigma\cup \axof{\I_\nbind}$, which proves that $J \in
\I_\nbind$.  Since $I \not \models C\purify{\sigma}\theta$ we have $J
\not \models C\sigma$, which is impossible because $J \in \I_\nbind$
and $S\sigma \models^{\nonbasetheory} C\sigma$.
\end{proof}

\newcommand{\extensible}{{$0$-extensible}}
\newcommand{\wcompact}{weakly compact}

\newcommand{\setofsorts}{\overline{\asort}}

\subsubsection{Completeness of $\Theta_{\bind}$ for \Gclauses}

In this section, we prove that any procedure that is \sadequateone\ is
also complete for some classes of sets of \emph{possibly infinite}
\gclauses\ -- this is of course not the case in general.  We first
notice that the notation $\deriv{S}$ of Definition \ref{def:deriv} can
be extended to \gclauses, by allowing infinite disjunctions:
\begin{definition}
  Given a set of clauses, $S$, we denote by $\derivinf{S}$ the set of
  \gclauses\ of the form $\{ C_i\sigma \mid i \in \mathbb{N}, C_i \in
  S, \text{$\sigma_i$ is a pure substitution} \}$.
\end{definition}
The notation $\instB{S}{G}$ also extends to \gclauses:
$\instB{S}{\abs}$ is the set of clauses $C\sigma$ such that $C \in S$
and $\sigma$ maps every variable in $C$ to a term in $\abs$.

\begin{proposition}
\label{prop:finiteb}
Let $S$ be a finite set of clauses and  $\abs$ be a finite set of terms.
Then $\instB{\derivinf{S}}{\abs}$ is a finite set of clauses.
\end{proposition}

\begin{proof}
  By definition, any literal occurring in $\derivinf{S}$ is of the form
  $L\sigma$ where $L$ is a literal occurring in a clause $C \in S$ and
  $\sigma$ is a pure substitution.  Thus any literal occurring in
  $\instB{\derivinf{S}}{\abs}$ is of the form $L\sigma\theta$ where
  $L$ is literal occurring in a clause in $S$, $\sigma$ is pure and
  $\theta$ maps every variable to a term in $\abs$.  Obviously, since
  $\abs$ and $S$ are finite, there are finitely many literals of this
  form. Hence all the \gclauses\ in $\instB{\derivinf{S}}{\abs}$ are
  actually finite, and there are only finitely many possible clauses.
\end{proof}

\begin{lemma}
\label{lem:lememb}
Let $S$ be a set of clauses and $S'$ a set of {\gclauses} with
$S' \subseteq \derivinf{S}$.  If $\abs$ if a finite set of terms, then
there exists a set of clauses $S'' \emb S'$ such that
$\instB{S''}{\abs} = \instB{S'}{\abs}$.
\end{lemma}

\begin{proof}
  Let $C$ be a clause in
  $\instB{S'}{\abs}$; by Proposition \ref{prop:finiteb}, $C$ is
  finite.  By definition there exists an \gclause\ $C' \in S'$ such
  that $C = C'\theta$, where $\theta$ is a substitution mapping all
  the variables in $\Var{C'}$ to a term in $\abs$.
  Every literal in $C'$ is of the form $L\gamma$, where literal $L$
  occurs in $S$ and $\gamma$ is a pure substitution of $\var(L)$.
  Since $S$ and $G$ are finite, there is a finite number of possible
  pairs $(L,\gamma\theta)$. Thus there exists a finite subset $D_C
  \subseteq C'$ such that for every literal $L\gamma$ occurring in
  $C'$, there exists a literal $L\gamma' \in D_C$ with $\gamma\theta =
  \gamma'\theta$.

  Every variable occurring in a literal $L\gamma$ of $C'$ is of the
  form $x\gamma$, where $x \in \Var{L}$.  Let $\eta_C$ be the
  substitution mapping every variable $x\gamma \in \Var{C' \setminus
    D_C}$ to $x\gamma'$.  Then for every literal $L\gamma \in C'$, we
  have $L\gamma\eta_C = L\gamma' \in D_C$.  Thus $C'\eta_C = D_C$;
  furthermore, $\eta_C$ is pure and $D_C\eta_C = D_C$.

  We define $S'' = \{ D_C \mid C \in\instB{S'}{\abs} \}$; obviously
  $S'' \emb S'$ and by definition $\instB{S''}{\abs} \supseteq
  \instB{S'}{\abs}$.  Conversely, let $E$ be a clause in
  $\instB{S''}{\abs}$, $E$ is necessarily of the form $D_C\theta$
  where $C \in \instB{S'}{\abs}$ and $\theta$ maps every variable to a
  term in $\abs$. But then $E$ is of the form $C'\eta_C\theta$, where
  $C'\in S'$, and $\eta_C\theta$ is a substitution mapping every
  variable in $C'$ to a term in $\abs$; thus $E$ must occur in
  $\instB{S'}{\abs}$.
\end{proof}

The next lemma proves  the completeness result for {\gclauses}:

\begin{lemma}
\label{lem:compgen}
Let $\Theta$ be a \sadequateone\ instantiation procedure and $S$ be a
set of clauses.  If $S' \subseteq \derivinf{S}$ then $S'$ and $\instB{S'}{\bs{S}}$ are
  $\basetheory$-equisatisfiable.
\end{lemma}

Note that the clauses in $S$ are finite, but those in $S'$ may be
infinite.

\begin{proof}
  $\instB{S'}{\bs{S}}$ is a logical consequence of $S'$, thus if $S'$
  is satisfiable then so is $\instB{S'}{\bs{S}}$; we now prove the
  converse.  Let $I$ be an interpretation validating
  $\instB{S'}{\bs{S}}$.  By Lemma \ref{lem:lememb}, there exists a set
  of clauses $S''$ such that $S'' \emb S'$ and $\instB{S'}{\bs{S}} =
  \instB{S''}{\bs{S}}$.  Since $I \models \instB{S'}{\bs{S}}$, we
  deduce that $\instB{S''}{\bs{S}}$ is satisfiable, hence (since by
  Condition \ref{theta1s:bs} in Definition \ref{def:theta1s}, $\Theta$
  is complete\footnote{Recall that $S''$ is a set of \emph{finite}
    clauses.}) so is $S''$. But $S'' \emb S'$ therefore by Proposition
  \ref{prop:logc}, $S'$ is satisfiable.
\end{proof}


\newcommand{\rri}{R_I}
\newcommand{\aai}{A_I}
\newcommand{\api}{B_I}
\newcommand{\eei}{E_I}
\newcommand{\uu}{U}

\subsubsection{Main Proof}\label{sec:mainpf}

We are now in the position to give the proof of the main theorem.

\begin{proof}[of Theorem \ref{theo:comp}]
  Let $\Theta \isdef \hcombinproc{\Theta_\bind}{\Theta_\nbind}$ and
  let $S$ be an unsatisfiable clause set in $\C$. We prove that
  $\Theta(S)$ is also unsatisfiable.

  Let $I \in \I_\bind$, by Lemma \ref{lem:ri}, the set
  $\nonbasepart{S}{I} = \{ \projecnonbase{C}\eta \mid C \in S, \eta
  \in \falsifyingsubs{I}{C} \}$ is $\nonbasetheory$-unsatisfiable, and
  by completeness of $\Theta_\nbind$, so is
  $\Theta_\nbind(\nonbasepart{S}{I})$. We define
  \[\aai\ =\ \Setof{C\eta\theta}{C\in S,\, \projecnonbase{C}\eta\theta
    \in \Theta_\nbind(\nonbasepart{S}{I})}.\] This set may be
  infinite, since no assumption was made on the decidability of
  $\nonbasetheory$. Every clause in $\aai$ is of the form
  $C\eta\theta$ where $I\not\models
  \projecbase{C}\eta$,\footnote{Recall that
    $\projecbase{C}\eta = \projecbase{C}\eta\theta$, since $\eta$ is a
    ground {\bsubstitution}} and by Proposition \ref{prop:decsubs},
  $C\eta\theta = C\sigma\sigma'$, where $\sigma$ is a {\tsubstitution}
  and $\sigma'$ is a {\bsubstitution}. In particular, since
  $\dom(\sigma) \subseteq \nonbasevar$, $\projecbase{C}\sigma\sigma' =
  \projecbase{C}\sigma'$ and $I\not\models \projecbase{C}\sigma'$.

  By construction, the set
  $\setof{\projecnonbase{C}\sigma\sigma'}{(\projecnonbase{C} \vee
    \projecbase{C})\sigma\sigma' \in \aai}$ is
  $\nonbasetheory$-unsatisfiable. Thus for every model $J$
   of $\aai$,
   there exists a clause $(\projecnonbase{C} \vee
    \projecbase{C})\sigma\sigma' \in \aai$ such that
    $J \not \models \projecnonbase{C}\sigma\sigma'$, hence
    $J \models \projecbase{C}\sigma\sigma'$ (since $J \models \aai$ we have $J \models (\projecnonbase{C} \vee
    \projecbase{C})\sigma\sigma'$). Since the $\projecbase{C}$
    cannot contain {\tvariable}s, we have $\projecbase{C}\sigma\sigma' = \projecbase{C}\sigma'$.
    Hence $\aai\models_{\nonbasetheory}
  \bigvee_{C\sigma\sigma'\in \aai} \projecbase{C}\sigma'$.
  We let $T = \projecbase{S}$ and define:
  \[\api\ =\ \Setof{C\sigma\purify{\sigma'}}{C\sigma\sigma' \in \aai}
  \textrm{ and } \eei\ =\ \bigvee_{C\sigma\sigma'\in
    \aai}\projecbase{C}\purify{\sigma'}.\] Note that since $\aai$ may
  be infinite, $\eei$ is an {\gclause} that belongs to
  $\derivinf{T}$. Lemma \ref{lem:gene} guarantees that $\api
  \entailsrigid{\nonbasetheory} \eei$; thus by definition, for all
  sets of ground {\bterm}s $\BB$, $\instB{\api}{\BB}
  \models_{\nonbasetheory} \instB{\eei}{\BB}$. This is in particular
  the case for $\BB = \BB_T$.

  Let $U = \setof{\eei}{I\in \I_\bind}$; by construction, for all $I
  \in \I_\bind$, $I\not\models U$; hence $U$ is
  $\basetheory$-unsatisfiable and since  $U \subseteq \derivinf{T}$, by
  Lemma \ref{lem:compgen}, $\instB{U}{\bs{T}}$ is also
  $\basetheory$-unsatisfiable. We have shown that $\instB{\api}{\BB}
  \models_{\nonbasetheory} \instB{\eei}{\BB}$. This, together with the fact that
  $\instB{U}{\bs{T}} = \bigcup_{I\in \I_\bind}\instB{\eei}{\bs{T}}$
  permits to deduce that $\bigcup_{I\in \I_\bind}\instB{\api}{\bs{T}}
  \models_{\nonbasetheory} \instB{U}{\bs{T}}$. Since $\instB{U}{\bs{T}}$ is $\basetheory$-unsatisfiable (hence also $\hcombin{\basetheory}{\nonbasetheory}$-unsatisfiable), $\bigcup_{I\in
    \I_\bind}\instB{\api}{\bs{T}}$ is
  $\hcombin{\basetheory}{\nonbasetheory}$-unsatisfiable.

  There remains to prove that $\bigcup_{I\in
    \I_\bind}\instB{\api}{\bs{T}} \subseteq \Theta(S)$ to obtain the
  result. Consider the function $\am$ that maps every term of a sort
  $\asort \in \base$ to $\speccst_\asort$; it is clear that
  $\am(\nonbasepart{S}{I}) \subseteq \projecnonbase{S}\sutop$. In
  particular, if $\projecnonbase{C}\sigma\sigma' \in
  \Theta_{\nonbasetheory}(\nonbasepart{S}{I})$, then by the {\preservation}
  and monotonicity of $\Theta_\nonbasetheory$,
  \[\projecnonbase{C}\sigma\purify{\sigma'}\sutop\ =\
  \am(\projecbase{C}\sigma\sigma')\ \in\
  \Theta_{\nonbasetheory}(\am(\nonbasepart{S}{I}))\ \subseteq\
  \Theta_{\nonbasetheory}(\projecnonbase{S}\sutop).\]
  Therefore, $\instB{(C\sigma\purify{\sigma'})}{\bs{T}} \subseteq \Theta(S)$,
  hence the result.

  The fact that $\hcombin{\Theta_\bind}{\Theta_\nbind}$ is
  \preserving\ and monotonic follows immediately from the definition
  and from the fact that $\Theta_\nbind$ is {\preserving} and that
  $\Theta_\bind$ and $\Theta_\nbind$ are monotonic.
\end{proof}

\section{Applications}

\label{sect:appl}

In this section, we show some examples of applications of Theorem
\ref{theo:comp} that are particularly relevant in the context of
program verification.

\subsection{Examples of \Sadequateone\ \Theories}

\subsubsection{Presburger Arithmetic}

\label{sect:pa}

\newcommand{\specconst}{\chi}

No \sadequateone\ instantiation procedure can be defined for the
\theory\ $\Tnat$ as defined in Section \ref{sect:ex}, as evidenced by
the following example.

\begin{example}
  Assume that a \sadequateone\ procedure $\Theta$ exists, and consider
  the clause set $S = \{ x \not \iseq y+1, y \not \iseq 0 \}$.  Since
  $\Theta$ is {\sadequateone} by hypothesis, by Condition
  \ref{theta1s:bs} of Definition \ref{def:theta1s}, $\Theta(S) =
  \instB{S}{\bs{S}}$ for some finite set of ground terms $\bs{S}$, and
  by Condition \ref{theta1s:disj}, $\bs{S}$ contains $\bs{\deriv{S}}$.
  But $\deriv{S}$ contains in particular the clause: $C_n:
  \bigvee_{i=1}^n x_i \not \iseq x_{i-1}+1 \vee x_0 \not \iseq
  0$. $C_n$ is obviously $\Tnat$-unsatisfiable, but the only instance
  of $C_n$ that is $\Tnat$-unsatisfiable is: $C_n \{ x_i \mapsto i
  \mid i \in [0,n] \}$.  Consequently $\{ i \mid i \in [0,n] \}
  \subseteq \bs{S}$ hence $\bs{S}$ cannot be finite, thus
  contradicting Condition \ref{theta1s:bs}.
\end{example}

It is however possible to define \sadequateone\ procedures for less
general \theories, that are still of a practical value.

\newcommand{\TT}{T_B}
\newcommand{\kk}{m}

\begin{definition}
\label{def:tnatok}
Let $\specconst$ be a special constant symbol of sort $\snat$,
 let $\kk$ be a natural number distinct from $0$ and let $\TT$ be a set of ground terms of sort $\snat$ not containing $\specconst$.
We denote by $\Tnatok$ the specification $(\Inat',\Cnatok)$ defined as
follows.  $\axof{\Inat'} \isdef \axof{\Inat} \cup \setof{\specconst
  > t+\kk}{t \in \TT}$, where $\axof{\Inat}$ is defined in
Example \ref{ex:presb} (Section \ref{sect:ex}). $\Cnatok$ contains every clause set $S$ such
that every non-ground
literal occurring in a clause in $S$ is of one of the following forms:
\begin{itemize}
\item{$x \not \leq t$ or $t \not \leq x$ for some variable $x$ and for
    some ground term $t \in \TT$;}
\item{$x \not \leq y$ for some variables $x,y$;}
\item{$x \not \iseq_k t$ for some $k \in \mathbb{N} \setminus \{ 0 \}$
    that divides $\kk$,
 some ground term $t \in \TT$ and  some variable $x$.}
 \end{itemize}
\end{definition}


Intuitively, the constant $\specconst$ occurring in $\axof{\Inat'}$ is
meant to translate the fact that the terms appearing in $S$ admit an
upper bound (namely $\specconst$). It is clear that if $S$ is an arbitrary set of arithmetic
clauses (not containing the special constant $\specconst$), then the set $\TT$ and the integer $\kk$ can be computed so
that $S$ indeed belongs to $\Cnatok$.


\newcommand{\Thetaz}{\Theta_\mathbb{Z}}
\newcommand{\Thetafol}{\Theta_{\text{fol}}}
\newcommand{\thetafol}{\Thetafol}

\newcommand{\Thetam}{\Theta_{\in}}
\newcommand{\bsz}[1]{G^\mathbb{Z}_{#1}}

\newcommand{\bsm}[1]{G^{\in}_{#1}}
\newcommand{\divise}[2]{\neg #2(#1)}


\begin{definition}
\label{def:bsz}
For every set of clauses $S \in \Cnatok$, let $B_S$ be the set of
ground terms $t$ such that either $t = \specconst$ or $S$ contains an
atom of the form $x \leq t$.
We define the instantiation procedure  $\Thetaz$ by:
 $\Thetaz(S) \isdef \instB{S}{\bsz{S}}$, where $\bsz{S}$ is defined by:
 $\bsz{S} \isdef \{ t-l \mid t \in B_S, 0 \leq l < \kk \}$.
\end{definition}


The two following propositions are straightforward consequences of the
 definition:

\begin{proposition}
\label{prop:monobsz}
If $S \subseteq S'$ then $\bsz{S} \subseteq \bsz{S'}$.
\end{proposition}

\begin{proposition}
\label{prop:disjbsz}
$\bsz{S} = \bsz{\deriv{S}}$.
\end{proposition}

\begin{proof}
  This is immediate because the set of ground terms occurring in
  $\deriv{S}$ is the same as that of $S$, since the atoms in
  $\deriv{S}$ are pure instances of atoms in $S$. Thus $B_{\deriv{S}}
  = B_S$.
\end{proof}

\begin{theorem}
\label{theo:natok}
$\Thetaz$ is \sadequateone\ if $\basetheory = \Tnatok$.
\end{theorem}

\begin{proof}
  We adopt the following notations for the proof: given a set of terms
  $W$, we write $x \not \leq W$ for $\bigvee_{t \in W} x \not \leq t$
  and $x \not \geq W$ for $\bigvee_{t \in W} x \not \geq t$.
  Additionally, if $K$ is a set of pairs $(k,t) \in \mathbb{N} \times
  \gt{\snat}$ then we denote by $\divise{x}{K}$ the disjunction
  $\bigvee_{(k,t) \in K} x \not \iseq_k t$.

  Let $S \in \Cnatok$ and
  assume that $S$ is $\Tnatok$-unsatisfiable, we prove that
  $\Thetaz(S)$ is also $\Tnatok$-unsatisfiable. Let $I \in \Inat'$,
  then in particular, $I \models \{ \specconst > t+\kk \mid
  \mbox{$t$ is a ground term in $S'$} \}$. Let $C$ be a clause in
  $S$ such that $I \not \models C$.  By definition of $\Cnatok$, $C$
  can be written as $C = D \vee \bigvee_{i=1}^n (x_i \not \leq U_i
  \vee x_i \not \geq L_i \vee \divise{x_i}{K_i})$, where $D$ is ground
  and where the $x_i$'s ($1 \leq i \leq n$) denotes distinct
  variables\footnote{Note that the sets $U_i$, $L_i$ and $K_i$ could
    be empty.}. Since $I\not\models C$, there exists a ground
  substitution $\theta$ such that $I\not\models C\theta$, i.e., for
  all $i \in [1,n]$:
  \begin{itemize}
  \item $\forall u \in U_i$, $\valueof{x_i\theta}{I} \leq \valueof{u}{I}$;
  \item $\forall l \in L_i$, $\valueof{l}{I} \leq \valueof{x_i\theta}{I}$;
  \item $\forall (k,t) \in K_i$, $\valueof{x_i\theta}{I} \iseq_k
    \valueof{t}{I}$.
  \end{itemize}
  If $\valueof{x_i\theta}{I}$ is such that $\valueof{x_i\theta}{I} >
  \valueof{\specconst}{I}$, then it is straightforward to verify that
  $\valueof{x_i\theta}{I}-\kk$ satisfies the same conditions,
  since for all terms $t$ in $U_i \cup L_i$, $\valueof{\specconst}{I}
  - \kk > \valueof{t}{I}$, and since $\kk$ is a common multiple of
  every $k$ occurring in $K_i$.  We may therefore
  assume that $\valueof{x_i\theta}{I} \leq \valueof{\specconst}{I}$.

  We denote by $u_i$ an element in $U_i \cup \set{\specconst}$ such
  that $\valueof{u_i}{I}$ is minimal in $\{ \valueof{u}{I} \mid u \in
  U_i\cup \set{\specconst} \}$, and by $m_i$ the greatest integer such
  that $m_i \leq \valueof{u_i}{I}$ and for every $(k,t) \in K_i$, $m_i
  \iseq_k t$ holds; the existence of $m_i$ is guaranteed by what
  precedes and $\valueof{x_i\theta}{I} \leq m_i$. We cannot have $m_i
  + \kk \leq u_i$, because otherwise $m_i$ would not be the greatest
  integer satisfying the conditions above. Thus, necessarily, $m_i >
  \valueof{u_i}{I}-\kk$, and there must exist a term $v_i \in \bsz{S}$
  such that $\valueof{v_i}{I} = m_i$.  Let $\sigma \isdef \{ x_i
  \mapsto v_i \mid i \in [1,n] \}$, we deduce that $I\not\models
  C\sigma$. Since $C\sigma \in \instB{S}{\bsz{S}}$, we conclude that
  $\instB{S}{\bsz{S}}$ is $\Tnatok$-unsatisfiable, hence the result.


  By construction, $\bsz{S}$ is finite, hence Condition
  \ref{theta1s:bs} of Definition \ref{def:theta1s} is satisfied.
  By Propositions \ref{prop:monobsz} and \ref{prop:disjbsz},
  Conditions \ref{theta1s:mono} and \ref{theta1s:disj} are satisfied,
  respectively, which concludes the proof.
\end{proof}


\newcommand{\upbnd}[2]{B_{#1}^{#2}}
\newcommand{\isless}{\preceq}
\newcommand{\mybsup}[3]{\bar{#1}_{#2}^{#3}}

\newcommand{\dZ}{\Tnat}
\newcommand{\modelz}{\models_{\dZ}}
\newcommand{\lfta}{\leftarrow}

\subsubsection{Term Algebra with Membership Constraints}

\newcommand{\tmem}{\T_{\in}}
\newcommand{\Imem}{\I_{\in}}
\newcommand{\Cmem}{\C_\in}
\newcommand{\ffs}{\Sigma}
\newcommand{\gtffs}{\gt{}(\Sigma)}
\newcommand{\sets}{\mathfrak{P}}
\newcommand{\vofs}[1]{\hat{#1}}

We give a second example of a \theory\ for which a \sadequateone\
instantiation procedure can be defined. We consider formul{\ae} built
over a signature containing:
 \begin{itemize}
 \item{a set of free function symbols $\ffs$;}
 \item{a set of constant symbols interpreted as ground terms built on
     $\ffs$;}
 \item{a set of monadic predicate symbols $\sets$, each predicate $p$
     in $\sets$ is interpreted as a (fixed) set $\vofs{p}$ of ground
     terms built on $\ffs$.  We assume that the emptiness problem is
     decidable for any finite intersection of these sets (for instance
     $\vofs{p}$ can be the set of terms accepted by a regular tree
     automaton, see \cite{tata97} for details).}
  \end{itemize}

From a more formal point of view:

\begin{definition}
  Let $\ffs \subseteq \funbase$. We denote by $\gtffs_\asort$ the set
  of ground terms of sort $\asort$ built on $\ffs$. Let $\sets$ be a finite
  set of unary predicate symbols, together with a function $p \mapsto
  \vofs{p}$ mapping every symbol $p: \asort \rightarrow \sbool \in
  \sets$ to a subset of $\gtffs_\asort$.

We denote by $\tmem$ the \theory\ $(\Imem,\Cmem)$ where:
\begin{itemize}
\item{$\axof{\Imem}$ contains the following axioms:
\[
\begin{tabular}{ll}
  $\bigvee_{t\in \gtffs_\asort} x \iseq t$ & for $\asort \in \base$, $x \in \Xbase$ ,\\
  $x_i \iseq y_i \vee f(x_1,\ldots,x_n) \not \iseq
  f(y_1,\ldots,y_n)$ & if  $f \in \ffs$, $i \in [1,n]$ \\
  $p(x) \vee t \not \in \vofs{p}$ & if $p \in \sets$, $t \in \vofs{p}$. \\
\end{tabular}
\]
}
\item{Every non-ground atom in $\Cmem$ is of the form $\neg p(x)$, or
    of the form $x \not \iseq t$ for some ground term $t$.}
    \end{itemize}
\end{definition}

The axioms of $\axof{\Imem}$ entail the following property which is
proved by a straightforward induction on the depth of the terms:

\begin{proposition}\label{prop:injmem}
  For all interpretations $I \in \Imem$ and all terms $t,t'$ occurring
  in a clause in $\Cmem$, if $\valueof{t}{I} = \valueof{t'}{I}$ then
  $t = t'$.
\end{proposition}

If the sets in $\{ \vofs{p} \mid p \in \sets \}$ are regular then
$\tmem$ is well-known to be decidable, see, e.g., \cite{CD94}. We
define the following instantiation procedure for $\tmem$:

\begin{definition}
Let $\bsm{S}$ be a set of ground terms containing:
\begin{itemize}
\item{Every ground term $t$ such that $S$ contains an atom of the form $x \not \iseq t$.}
\item{An arbitrarily chosen ground term $s_P \in \bigcap_{p \in P}
    \vofs{p}$, for each $P \subseteq \sets$ such that $\bigcap_{p \in
      P} \vofs{p} \not = \emptyset$ (recall that the emptiness
    problem is assumed to be decidable).}
\end{itemize}
Let $\Thetam \isdef \instB{S}{\bsm{S}}$.
\end{definition}

\begin{theorem}
$\Thetam$ is \sadequateone\ if $\basetheory = \tmem$.
\end{theorem}

\begin{proof}
  Let $C$ be a clause in $\Cmem$, $C$ is of the form $\bigvee_{i=1}^n
  x_i \not \iseq t_i \vee \bigvee_{i=1}^m \neg p_i(y_i) \vee D$ where
  $D$ is ground, $x_i$ and $y_j$ ($i \in [1,n]$, $j \in [1,m]$) are
  variables, $t_i$ is a ground term for $i \in [1,n]$ and $p_j \in
  \sets$ for $j \in [1,m]$.  Let $X = \{ x_1,\ldots,x_n \}$ and $Y =
  \{ y_1,\ldots,y_m \}$; note that these sets are not necessarily
  disjoint.  For every variable $y \in Y$ we denote by $P_y$ the set
  of predicates $p_j$ ($1 \leq j \leq m$) such that $y_j = y$ and we
  let $s_y \isdef s_{P_y}$.  Consider the substitution $\sigma$ of domain
  $X \cup Y$ such that:
\begin{itemize}
\item{$x_i\sigma \isdef t_i$ for every $i \in [1,n]$;}
\item{if $y \in Y \setminus X$ then $y\sigma \isdef s_y$ (notice that $s_y$ must be defined since $y \in Y$)
}
\end{itemize}
We prove that $C\sigma \models_{\tmem} C$.

Let $ I$ be an interpretation such that $I \models C\sigma$ and $I
\not \models C$. Then there exists a substitution $\theta$ such that
$I \not \models C\theta$, which implies that for all $i \in [1,n]$,
$\valueof{x_i\theta}{I} = \valueof{t_i}{I}$, and for all $j \in
[1,m]$, $\valueof{y_j\theta}{I} \in \valueof{\vofs{p_j}}{I}$.
Proposition \ref{prop:injmem} entails that $x_i\theta = t_i$ for all
$i\in [1,n]$, and $y_j\theta \in \vofs{p_j}$ for all $j\in [1,m]$.
Thus, in particular, for all $x \in X$, $x\sigma = x\theta$, and for
all $y \in Y\setminus X$, $\bigcap_{p \in P_{y}} \vofs{p} \neq
\emptyset$.

Since $I \models C\sigma$ and $x_i\sigma = t_i$ for all
$i \in [1,n]$, there must exist a $j \in [1,m]$ such that
$\valueof{y_j\sigma}{I} \not \in \valueof{\vofs{p_j}}{I}$; and, again
by Proposition \ref{prop:injmem}, this is equivalent to $y_j\sigma
\notin \vofs{p_j}$.  If $y_j \in X$, then $y_j\theta = y_j\sigma
\notin \vofs{p_j}$
and $I\models C\theta$, which is impossible.  Thus $y_j \in Y
\setminus X$, and since $\bigcap_{p \in P_{y_j}} \vofs{p} \neq
\emptyset$, by construction, $y_j\sigma = s_{y_j} \in \vofs{p_j}$;
this contradicts the assumption that $y_j\sigma \notin \vofs{p_j}$.

Since $C\sigma \models_{\tmem} C$, we deduce that for every clause $C
\in S$, there exists a $D \in \instB{S}{\bsm{S}}$ such that $D
\models_{\tmem} C$, and therefore, $S \equiv_{\tmem}
\instB{S}{\bsm{S}}$.  By construction, $\bsm{S}$ is finite, $\bsm{S} =
\bsm{\deriv{S}}$ and $\bsm{S} \subseteq \bsm{S'}$ if $S \subseteq S'$.
Hence all the conditions of Definition \ref{def:theta1s} are satisfied.
\end{proof}

\newcommand{\anF}{{\cal F}}
\newcommand{\voff}[1]{\hat{#1}}
\newcommand{\spec}{a definition set}
\newcommand{\almostground}{essentially  ground}
\newcommand{\grt}[2]{\langle #1\rangle_{#2}}
\newcommand{\fterm}[2]{\phi_{#1}(#2)}

\subsection{Combination of \Theories}

Building on the results of the previous section, we now provide some concrete applications of Theorem \ref{theo:comp}.

\subsubsection{Combining First-order Logic without Equality and Presburger Arithmetic}

\label{sect:fol}

We begin with a simple example to illustrate how the method works.  We
show how to enrich the language of first-order predicate logic with
some arithmetic constraints.  We assume that $\F$ contains no function
symbol of co-domain $\snat$ other than the usual symbols $0,s,+,-$
introduced in Section \ref{sect:ex}.

Let $\Tfolne$ be the restriction of the \theory\ $\Tfol$ defined in
Example \ref{ex:fol} to non-equational clause sets (i.e. to clause
sets in which all atoms are of the form $t \iseq \true$).  We consider
the combination $\hcombin{\Tnatok}{\Tfolne}$ of the \theory\ $\Tnatok$
introduced in Section \ref{sect:pa} with $\Tfolne$.  According to
Theorem \ref{theo:natok}, $\Thetaz$ is \sadequateone\ for $\Tnatok$;
thus, in order to apply Theorem \ref{theo:comp}, we only need to find
a \adequatetwo\ instantiation procedure for $\Tfolne$.  We will use an
instantiation procedure based on \emph{hyper linking} \cite{LP92}. It
is defined by the following inference rule:

\begin{tabular}{c}
$\bigvee_{i=1}^n l_i, m_1 \vee C_1, \ldots, m_n \vee C_n$ \\
\hline
 $\bigvee_{i=1}^n l_i\sigma$
 \end{tabular}
 if $\sigma$ is an mgu. of the $(l_i,m_i^c)$'s.

If $S$ is a set of clauses, we denote by $\thetafol'(S)$ the set of clauses that can be obtained from $S$ by applying the rule above (in any number of steps) and by $\thetafol(S)$ the set of clauses obtained from $\thetafol(S)$ by replacing all remaining variables of sort $\asort$ by a constant symbol $\bot_\asort$ of the same sort.

\begin{proposition}
$\thetafol$ is \adequatetwo\ for $\Tfolne$.
\end{proposition}

\begin{proof}
  In \cite{LP92}, it is proven that $S$ and $\thetafol(S)$ are
  equisatisfiable, thus Condition \ref{theta2:comp} of Definition
  \ref{def:theta2} holds; furthermore, by definition, $\thetafol$ is
  monotonic.  To verify that $\thetafol$ is \preserving, it suffices to
  remark that if a clause $D$ is deducible from a set of clauses $S$
  by the instantiation rule above, then for every \bmapping\ $\am$,
  $\am(D)$ must be deducible from $\thetafol(\am(S))$, since the unifiers
  are not affected by the replacement of ground terms: if an mgu maps
  a variable $x$ to a term $t$ in $S$, then the corresponding mgu will
  map $x$ to $\am(t)$ in $\am(S)$.
 \end{proof}

 Theorem \ref{theo:comp} guarantees that
 $\hcombinproc{\Thetaz}{\thetafol}$ is \complete\ for
 $\hcombin{\Tnatok}{\Tfolne}$.  Note that in general,
 $\hcombinproc{\Thetaz}{\thetafol}$ (and $\thetafol$) are not terminating.
 However, $\hcombinproc{\Thetaz}{\thetafol}$ is terminating if the set of ground
 terms containing no subterm of sort $\snat$ (and distinct from
 $\speccst_\snat$) is finite (for instance if $\F$ contains no function
 symbol of arity greater than $0$ and of a sort distinct from
 $\snat$).

\begin{example}
  Consider the following set of clauses $S$, where $i,j$ denote variables of
  sort $\snat$, $x,y$ denote variables of sort $\asort$, and $\F$
  contains the following symbols: $a,b: \snat$, $c,d: \asort$, $p:
  \snat \times \asort \rightarrow \sbool$ and $q: \snat \times \asort
  \times \asort \rightarrow \sbool$.
\[
\begin{tabular}{ll}
($1$) & $\neg p(i,x) \vee \neg q(i,y) \vee r(i,x,y)$ \\
($2$) & $p(a,c)$ \\
($3)$ & $j \not < b \vee q(j,d)$ \\
($4)$ & $i \not \iseq_2 0 \vee \neg r(i,x,y)$ \\
\end{tabular}
\]

Clauses ($2$) and ($3$) are not in $\T$. Indeed, the non-arithmetic
atom $p(a,c)$ contains a non-variable arithmetic subterm $a$ and ($3)$
contains a literal $j \not < b$ that is not allowed in $\Tnatok$ (see
Definition \ref{def:tnatok}).  Thus these clauses must be reformulated
as follows:

\[
\begin{tabular}{ll}
($2$)' & $i \not \leq a \vee a \not \leq i \vee p(i,c)$ \\
($3$)' & $j \not \leq b-1 \vee q(j,d)$ \\
\end{tabular}
\]

To apply the procedure $\hcombinproc{\Thetaz}{\thetafol}$, we compute the
set $\projecnonbase{S}$ and replace every arithmetic variable occurring
in it by a special constant $\botz$ of sort $\snat$:
\[ \projecnonbase{S} = \left\{
\begin{tabular}{l}
$\neg p(\botz,x) \vee \neg q(\botz,y) \vee r(\botz,x,y)$ \\
$p(\botz,c)$ \\
$q(\botz,d)$ \\
$\neg r(\botz,x,y)$ \\
\end{tabular}\right.
\]

We apply the procedure $\thetafol$. The reader can verify that we obtain
the following clause set:

\[
\thetafol(\projecnonbase{S}) = \left\{
\begin{tabular}{lll}
$\neg p(\botz,\bot) \vee \neg q(\botz,\bot) \vee r(\botz,\bot,\bot)$ \\
$p(\botz,c)$ \\
$q(\botz,d)$ \\
$\neg r(\botz,\bot,\bot)$ \\
$\neg p(\botz,c) \vee \neg q(\botz,d) \vee r(\botz,c,d)$ \\
$\neg r(\botz,c,d)$ \\
\end{tabular}\right.
\]

Next we consider the clauses in $\projecbase{S}$: $\{ i \not \leq a
\vee a \not \leq i, j \not \leq b-1, i \not \iseq_2 0 \}$ and compute
the set $\bsz{\projec{S}{\Tnatok}}$, according to Definition
\ref{def:bsz}.  The terms occurring as the right operands of a symbol
$\leq$ are $\{ a,b-1 \}$.  The least common multiple of all the
natural numbers $k$ such that $\projecbase{S}$ contains a comparison
modulo $k$ is $2$.  Thus $\bsz{\projec{S}{\Tnatok}} = \{ a,b-1,a-1,b-2
\}$. To get the clause set $\hcombinproc{\Thetaz}{\Theta}(S)$, the
substitutions generated by $\Theta$ are combined with all
instantiations of integer variables by elements of
$\bsz{\projec{S}{\Tnatok}}$. This yields: {\small
\[
\begin{tabular}{ll}
$\neg p(a,\bot) \vee \neg q(a,\bot) \vee r(a,\bot,\bot)$ & $p(a,c)$
 \\
$\neg p(b-1,\bot) \vee \neg q(b-1,\bot) \vee r(b-1,\bot,\bot)$ & $p(b-1,c)$ \\
$\neg p(a-1,\bot) \vee \neg q(a-1,\bot) \vee r(a-1,\bot,\bot)$ & $p(a-1,c)$ \\
$\neg p(b-2,\bot) \vee \neg q(b-2,\bot) \vee r(b-2,\bot,\bot)$ &
$p(a-2,c)$ \\
$\neg r(a,\bot,\bot)$ &
$\neg r(a,c,d)$ \\
$\neg r(b-1,\bot,\bot)$ &
$\neg r(b-1,c,d)$ \\
$\neg r(a-1,\bot,\bot)$ &
$\neg r(a-1,c,d)$ \\
$\neg r(b-2,\bot,\bot)$ &
$\neg r(b-2,c,d)$ \\
$\neg p(a,c) \vee \neg q(a,d) \vee r(a,c,d)$ & $q(a,d)$ \\
$\neg p(b-1,c) \vee \neg q(b-1,d) \vee r(b-1,c,d)$ & $q(b-1,d)$\\
$\neg p(a-1,c) \vee \neg q(a-1,d) \vee r(a-1,c,d)$ & $q(a-1,d)$ \\
$\neg p(b-2,c) \vee \neg q(b-2,d) \vee r(b-2,c,d)$ & $q(b-2,d)$ \\
\end{tabular}
\]
} The resulting set of clauses is
$\hcombin{\Tnatok}{\Tfolne}$-unsatisfiable, hence, so is $S$.
\end{example}

\subsubsection{Arrays with Integer Indices and Uninterpreted Elements}


The \theory\ of arrays with integer indices and uninterpreted elements
can be defined as a \nestedcombin\ of the \btheory\ $\Tnatok$ defined
in Section \ref{sect:pa} with a simple \theory\ $\Tearrays =
(\Iall,\Cearrays)$, where the clauses in $\Cearrays$ are built on a
set of variables of sort $\snat$, on a signature containing only
constant symbols of sort $\sarray$ or $\selem$ and a function symbol
$\select: \sarray \times \snat \rightarrow \selem$.
We have assumed that $\Cearrays$ contains no occurrence of the
function symbol $\store$ for convenience. There is no loss of
generality: indeed, every definition of the form $s = \store(t,i,a)$
where $s,t,i,a$ are ground terms can be written as the conjunction of
the following clauses:
\[
\begin{tabular}{l}
$\select(s,i) = v$ \\
$i+1\not \leq z \vee \select(s,z) \iseq \select(t,z)$ \\
$z \not \leq i-1  \vee \select(s,z) \iseq \select(t,z)$ \\
\end{tabular}
\]
It is simple to verify that these three clauses are in
$\Cearrays$. Obviously, the last two clauses are equivalent to $z
\iseq i \vee \select(s,z) \iseq \select(t,z)$.


There exists a straightforward \adequatetwo\ instantiation procedure
for $\Tearrays$: namely the identity function $\id(S) \isdef S$. This
is indeed an instantiation procedure since all the variables occurring
in $\Cearrays$ are of type $\snat$; these variables will already be
instantiated by the instantiation procedure for $\Tnatok$ and the
remaining clause set will be ground.
The following result is a direct consequence of Theorem \ref{theo:comp}:

\begin{proposition}
$\hcombinproc{\Thetaz}{\id}$ is \complete\ for $\hcombin{\Tnatok}{\Tearrays}$.
\end{proposition}

We provide some examples of properties that have been considered in
\cite{Bradleybook,DBLP:conf/fossacs/HabermehlIV08,GNRZ07}, and can be expressed in
$\hcombin{\Tnatok}{\Tearrays}$ ($t$,$t'$ denotes constant symbols of
sort \sarrays).

\newcommand{\com}[1]{\multicolumn{2}{l}{\texttt{-- #1}}}
\newcommand{\mapfun}[2]{\mathrm{map}(#1,#2)}

{\small
\[
\begin{tabular}{ll}
($1$) & $\forall i,\, a \not \leq i \vee i \not \leq b \vee \select(t,i) \iseq v$ \\
\com{$t$ is constant on $[a,b]$.} \\
($2$) & $\forall i,\, a \not \leq i \vee i \not \leq b \vee \select(t,i) \iseq \select(t',i)$ \\
\com{$t$ and $t'$ coincide on $[a,b]$.} \\
($3$) & $\forall i,j,\, a \not \leq i \vee i \not \leq b \vee \vee c \not \leq j \vee j \not \leq d \vee \select(t,i) \not \iseq  \select(t',j)$ \\
\com{The restriction of $t$ and $t'$ to $[a,b]$ and $[c,d]$ respectively are disjoint.} \\
($4$) & $\forall i,j,\, i \not \iseq_2 0 \vee j \not \iseq_2 1 \vee \select(t,i) \not \iseq \select(t,j)$ \\
\com{The values of $t$ at even indices are disjoint from the ones at odd ones.} \\
($5$) & $\forall i,\, i \not \iseq_2 0 \vee \select(t,i) \iseq \select(t',i) \vee \select(t,i) \iseq \select(t'',i)$ \\
\com{For every even index, the value of $t$ is equal to the value of $t'$ or $t''$.} \\
($6$) & $\forall i,\, i \not \geq 0 \vee i \not \leq d \vee \select(t,i) \not \iseq \bot$ \\
 & $\forall i,\, i \not\geq \su(d) \vee \select(t,i) \iseq \bot$ \\
\com{Array $t$ has dimension $d$.} \\
($7$) & $\forall i,\, \select(\mapfun{f}{t},i) \iseq f(\select(t,i))$ \\
\com{Array $\mapfun{f}{t}$ is obtained from $t$ by iterating function $f$.} \\
\end{tabular}
\]
}

Properties ($1$-$3$) can be expressed in the \emph{Array property
  fragment} (see \cite{Bradleybook}), but not Property ($4$), because
of  condition $i \iseq_2 0$.  Property ($4$) is expressible in the
\emph{Logic for Integer Arrays} (LIA) introduced in
\cite{DBLP:conf/fossacs/HabermehlIV08}, but not Property ($5$), because
there is a disjunction in the value formula.

On the other hand, Properties such as Injectivity cannot be expressed
in our setting:
{\small
\[
\begin{tabular}{ll}
($8$)  & $\forall i,j,\, i \iseq j \vee  \select(t,i) \not \iseq \select(t,j)$ \\
\com{$t$ is injective.} \\
($9$)  & $\forall i,j,\, i \iseq j \vee  \select(t,i) \not \iseq \select(t,j) \vee \select(t,i) \iseq \bot$ \\
\com{$t$ is injective on its domain.} \\
\end{tabular}
\]}
Indeed, the literal $i \iseq j$ is not allowed in $\Cnatok$.


\subsubsection{Arrays with Integer Indices and Interpreted Elements}

\label{sect:intint}
Instead of using the mere \theory\ $\Tearrays$, one can combine the
\theory\ $\Tnatok$ with a richer \theory, with function and predicate
symbols operating on the elements of the arrays.  For instance,
consider the \theory\ $\Tearraysreals = (\Ireals,\Cearraysreals)$,
where $\axof{\Ireals}$ is some axiomatization of real closed fields
over a signature $\F_\mathbb{R}$ and the clauses occurring in
$\Cearraysreals$ are built on a set of variables of sort $\snat$ and
on a signature containing all function symbols in $\F_\mathbb{R}$,
constant symbols of sort $\sarray$ or $\sreal$ and a function symbol
$\select: \sarray \times \snat \rightarrow \sreal$.  Then
$\hcombin{\Tnatok}{\Tearraysreals}$ is the
\theory\ of {arrays with integer indices and real elements}, and an
immediate application of Theorem \ref{theo:comp} yields:

\begin{proposition}
$\hcombinproc{\Thetaz}{\id}$ is \complete\ for $\hcombin{\Tnatok}{\Tearraysreals}$.
\end{proposition}


To model arrays with integer indices and integer elements, it is
necessary to use a combination of the \theory\ $\Tnatok$ with a
\theory\ containing the symbols in $\Tnatok$: $0: \snat$, $s: \snat
\rightarrow \snat$, $\leq: \snat \times \snat \rightarrow \sbool$,
etc.  However, this is \emph{not} permitted in our approach since the
clause sets of the \ttheory\ would contain function symbols whose
co-domain would be a sort of the {\btheory} (namely $\snat$), thus
contradicting the conditions on $\base$ and $\nonbase$ (see Section
\ref{sect:bdef}).  A solution is to use a \emph{copy} of
the sort $\snat$ and of every symbol of co-domain $\snat$.  We denote
by $\Tearraysint$ the \theory\ $(\Inat',\Carraysint)$ where
$\axof{\Inat'}$ is the image of $\axof{\Inat}$ by the previous
transformation and where the clause sets in $\Cearraysint$ are built
on a set of variables of sort $\snat$ and on a signature containing
all function symbols $0',s',\leq'$,\ldots in $\axof{\Inat'}$, constant
symbols of sort $\sarray$ or $\snat'$ and a function symbol $\select:
\sarray \times \snat \rightarrow \snat'$.  Then
$\hcombin{\Tnatok}{\Tearraysint}$ is a \theory\ of {arrays with
  integer indices and integer elements}, and by Theorem
\ref{theo:comp}, $\hcombinproc{\Thetaz}{\id}$ is \complete\ for
$\hcombin{\Tnatok}{\Tearraysint}$.

Note however that, due to the fact that the sort symbols are renamed,
equations between integer elements and integer indices are not
permitted: indices cannot be stored into arrays and terms of the form
$\select(t,\select(t,i))$ are forbidden.  However, the sharing of a
constant symbol $c$ between the two sorts $\snat$ and $\snat'$ (as in
the equation: $\select(t,c)\iseq c$) is possible, by adding ground
axioms of the form: $k \iseq c \Rightarrow k' \iseq c'$, where $c'$
denotes the copy of $c$, $k$ is any integer in $\snat$ and $k'$
denotes its copy in $\snat'$. Let $A$ denote this set of axioms; it is
obvious that $A$ is countably infinite. It is clear that
$\hcombinproc{\Thetaz}{\id}(S\cup A) = \hcombinproc{\Thetaz}{\id}(S)
\cup A$, so that the instantiation procedure is not affected by this
addition. Thus these axioms can be simply removed afterward by
``merging'' $\snat$ and $\snat'$ and by replacing $c'$ by $c$ (it is
straightforward to verify that this transformation  preserves
satisfiability).

We provide some examples. $\leq'$ and $+'$ are renaming of the symbols
$\leq$ and $+$ respectively. Notice that the indices of the arrays are
of sort $\snat$, whereas the elements are of sort $\snat'$. The
following properties can be expressed in
$\hcombin{\Tnatok}{\Tearraysint}$:

 \[
\begin{tabular}{ll}
($1$) & $\forall i,j, i \not \leq j \vee \select(t,i) \leq' \select(t,j)$ \\
\com{$t$ is sorted.} \\
($2$) & $\forall i,j, a \not \leq i \vee i \not \leq b \vee c \not \leq j \vee j \not \leq c \vee \select(t,i) \leq' \select(t',j)$ \\
\com{The values of $t$ at $[a,b]$ are lower than the ones of $t'$ at $[c,d]$.} \\
$(3)$ & $\forall i, i \not \iseq_2 0 \vee i \not \leq n \vee \select(t,i) \iseq' \select(t',i)+'\select(t'',i)$ \\
\com{For every even index lower than $n$, $t$ is the sum of $t'$ and $t''$.}
\end{tabular}
\]

Here are some examples of properties that \emph{cannot} be handled:
\[
\begin{tabular}{ll}
($4$) & $\forall i, \select(t,i) \iseq i$ \\
\com{$t$ is the identity.} \\
($5$) & $\forall i, \select(t,i) - \select(t,i+1) \leq 2$ \\
\com{The distance between the values at two consecutive index}\\
 \com{is at most $2$.}
\end{tabular}
\]
Property ($4$) is not in $\hcombin{\Tnatok}{\Tearraysint}$ because
there is an equation relating an element of sort $\snat$ (i.e. an
index) to an element of sort $\snat' \not= \snat$ (an element).
Property ($5$) could be expressed in our setting as $\forall i,j,\, j \not
\iseq i+1 \vee \select(t,i) - \select(t,j) \leq 2$ but  the atom
$j \not \iseq i+1$ is not in $\Tnatok$.  Property ($5$) can be
expressed in the logic LIA (see \cite{DBLP:conf/fossacs/HabermehlIV08}). This shows that the expressive power of this logic
is not comparable to ours.

These results extend straightforwardly to multidimensional arrays.

\subsubsection{Arrays with Translations on Arrays Indices}

In some cases, properties relating the value of an array at an index
$i$ to the value at index $i+k$ for some natural number $k$ can be
expressed by reformulations.

\newcommand{\amap}{\lambda}
\newcommand{\shiftable}{shiftable}

\begin{definition}
\label{def:shift}
Let $S$ be a clause set, containing clauses that are pairwise
variable-disjoint.  Let $\amap$ be a function mapping every array
constant to a ground term of sort $\snat$.  $S$ is \emph{\shiftable\
  relatively to $\amap$} iff the following conditions hold:

\begin{enumerate}
\item{Every clause in $S$ is of the form $C \vee D$, where $D$ is a
    clause in $\Tearraysint$ and every literal in $C$ is of one of the
    following form: $i \not \leq j+s$, $i \not \leq s$, $s \not \leq
    i$, $i \not \iseq_k s$, where $i,j$ are variables of sort $\snat$,
    $s$ is a ground term of sort $\snat$ and $k$ is a natural number.}
\item{For every clause $C \in S$ and for every literal $i \not \leq
    j+s$ occurring in $C$, where $i,j$ are variables and $s$ is a term
    of sort $\snat$, $C$ contains two terms of the form $\select(t,i)$
    and $\select(t',j)$ where $\amap(t') - \amap(t)$ is equivalent to
     $s$. \label{shift:c1}}
\item{If $C$ contains two terms of the form $\select(t,i)$ and
    $\select(t',i)$ then $\amap(t) = \amap(t')$. \label{shift:c2}}
\item{If $C$ contains a equation $t \iseq t'$ between arrays then
    $\amap(t) = \amap(t')$. \label{shift:c3}}
\end{enumerate}
\end{definition}

The existence of such a function $\amap$ is easy to determine:
conditions ($2$-$4$) above can immediately be translated into
arithmetic constraints on the $\amap(t)$'s, and the satisfiability of this
set of constraints can be tested by using any decision procedure for
Presburger arithmetic.

\newcommand{\shift}[1]{\mathrm{shift}(#1)}

We define the following transformation of clause sets:

\begin{definition}
  Let $t \mapsto t'$ be an arbitrarily chosen function mapping all the
  constants $t$ of sort $\sarrays$ to pairwise distinct fresh
  constants $t'$ of sort $\sarrays$.  We denote by $\shift{S}$ the
  clause set obtained from $S$ by applying the following rules:
\begin{itemize}
\item{every clause $C$ containing a term of the form $\select(t,i)$
    (where $i$ is a variable) is replaced by
$C \{ i \mapsto i-\amap(t) \}$;}
\item{then, every term of the form $\select(t,s)$ is replaced by
$\select(t',s+\amap(t))$.}
\end{itemize}
\end{definition}

\begin{lemma}
\label{def:st}
Let $S$ be a \shiftable\ clause set. Then:

\begin{itemize}
\item{$\shift{S}$ and $S$ are equisatisfiable.}
\item{$\shift{S}$ is in $\hcombin{\Tnatok}{\Tearraysint}$.}
\end{itemize}
\end{lemma}

\begin{proof}
  It is clear that for every clause $C$ in $S$, $C \equiv C \{ i
  \mapsto i-k \}$: since $i$ ranges over all integers, $i$ and $i-k$
  range over the same set.  The replacement of $\select(t,s)$ by
  $\select(t',s+\amap(t))$ obviously preserves sat-equivalence: it
  suffices to interpret $t'$ as the array defined by the relation:
  $\select(t',i) \isdef \select(t,i-\amap(t))$. Thus $\shift{S}$ and
  $S$ are equisatisfiable.

  We prove that $\shift{S}$ is in $\hcombin{\Tnatok}{\Tearraysint}$.
  By Condition \ref{shift:c2} of Definition \ref{def:shift}, if a
  clause $C\{ i \mapsto i-\amap(t) \}$ contains a term of the form
  $\select(s,i-\amap(t))$ then we must have $\amap(s) = \amap(t)$,
  thus this term is replaced by $\select(s',i)$ when the second rule
  above is applied. Consequently, the non-arithmetic part of the
  resulting clause cannot contain any non-variable term of sort
  $\snat$.  Now assume that $C$ contains an arithmetic literal of the
  form $i \leq j+s$. Then by condition \ref{shift:c1}, $C$ also
  contains terms of the form $\select(t,i)$ and $\select(t',j)$, where
  $\amap(t') - \amap(t)$ is equivalent to $s$.  Hence, the clause in
  $\shift{S}$ corresponding to $C$ contains the literal $i-\amap(t)
  \leq j-\amap(t')+s \equiv i \leq j - (\amap(t') - \amap(t))+s \equiv
  i\leq j$.
\end{proof}

We provide an example in which this result applies.

\begin{example}
Consider for instance the following clause set:
\[
S = \left\{\begin{tabular}{ll}
($1$) & $\forall i,j,\, a \not \leq i \vee i \not \leq b \vee j \not \iseq i-a \vee \select(s,i) \iseq \select(t,j)$ \\
\com{$s$ is identical to $t$ up to a shift of length $a$.} \\
($2$) & $\forall i,j,\, a \not \leq i \vee i \not \leq b \vee j \not \iseq i-a \vee \select(u,i) \iseq \select(s,j)$ \\
\com{$u$ is identical to $s$ up to a shift of length $a$.} \\
($3$) & $c \geq a+a$ \\
($4$) & $c \leq b$ \\
($5$) & $i \not \iseq c \vee j \not \iseq c-a-a \vee \select(u,c) \not \iseq \select(t,j)$ \\
\com{$u$ is not identical to $t$ up to a shift of length $a+a$.} \\
\end{tabular}\right.
\]
It is simple to check that $S$ is \shiftable\ relatively to the mapping:
$\amap(u) = a+a$, $\amap(s) = a$ and $\amap(t) = 0$.
According to Definition \ref{def:st}, $S$ is reformulated as follows:

\[
\shift{S} = \left\{\begin{tabular}{ll}
($1'$) & $\forall i, 0 \not \leq i \vee i \not \leq b-a \vee j \not \iseq i \vee \select(s',i) \iseq \select(t',j)$ \\
($2'$) & $\forall i, 0 \not \leq i \vee i \not \leq b-a \vee j \not \iseq i \vee \select(u',i) \iseq \select(s',j)$ \\
($3$) & $c \geq a+a$ \\
($4$) & $c \leq b$ \\
($5$) & $i \not \iseq c \vee j \not \iseq c-a-a \vee \select(u',c) \not \iseq \select(t',j)$ \\
\end{tabular}\right.
\]
$\shift{S}$ and $S$ are equisatisfiable, and $\shift{S}$ belongs to
$\hcombin{\Tnatok}{\Tearraysint}$. The unsatisfiability of $\shift{S}$
can be proven by applying the procedure $\hcombinproc{\Thetaz}{\id}$.
\end{example}

\subsubsection{Nested Arrays}

An interesting feature of this approach is that it can be applied
recursively, using as base and/or \target\ \theories\ some nested
combination of other \theories.

We denote by $\Tnatok'$ a copy of the specification $\Tnatok$ in which
the symbols $\snat$, $0$, $s$, $\leq$, \ldots are renamed into
$\snat'$, $0'$, $s'$, $\leq'$, \ldots We denote by $\Thetaz'$ the
corresponding instantiation procedure, as defined by Definition
\ref{def:bsz}.  Let $\Tearraysint'$ be a copy of the \theory\
$\Tearraysint$, in which the symbols $\snat'$, $0'$, $s'$, $\leq'$,
$\select$\ldots are renamed into $\snat''$, $0''$, $s''$, $\leq''$,
$\select'$ \ldots
 Let $\Tarraysintintint \isdef \hcombin{\Tnatok}{\hcombin{\Tnatok'}{\Tearraysint'}}$.

\begin{proposition}
  $\hcombinproc{\Thetaz}{\hcombinproc{\Thetaz'}{\id}}$ is complete for
  $\Tarraysintintint$.
\end{proposition}


In $\Tarraysintintint$, the
(integer) indices of an array $t$  can themselves be stored into arrays
of integers, \emph{but of a different type than $t$}.

\begin{example}
  The following clause set is $\Tarraysintintint$-unsatisfiable (for
  the sake of readability we use $t \not \iseq s$ as a shorthand for
  $t \not \leq s \vee t \not \leq s$):

\[
\begin{tabular}{ll}
($1$) & $i \leq j \vee \select(t,i) \leq \select(t,j)$ \\
\com{$t$ is sorted.} \\
($2$) & $i' \leq j' \vee \select'(t',i') \leq' \select'(t',j')$ \\
\com{$t'$ is sorted.} \\
($3$) &  $a \leq b$ \\
($4$) & $x \not \iseq a \vee y \not \iseq b \vee x' \not \iseq \select(t,x) \vee y'  \not \iseq \select(t,y)$ \\
& \hspace*{1cm} $\vee \select'(t',x') > \select(t',y')$. \\
\com{$t' \circ t$ is not sorted.}
\end{tabular}
\]


We describe the way the procedure works on this very simple but
illustrative example.  According to the definition of
$\hcombinproc{\Thetaz}{\hcombinproc{\Thetaz'}{\id}}$, the variables
$i$, $j$, $x$ and $y$ are replaced by a special symbol $\speccst$ and
the instantiation procedure $\hcombinproc{\Thetaz'}{\id}$ is applied.
The variables $i',j',x',y'$ are replaced by a constant symbol
$\speccst'$ and the procedure $\id$ is applied on the resulting clause
set (in a trivial way, since this set is ground).  Next, we apply the
procedure $\Thetaz'$.  According to Definition \ref{def:bsz},
$\Thetaz'$ instantiates the variables $i',j',x',y'$ by
$\select(t,\speccst)$.  This substitution is applied to the original
clause set and the procedure $\Thetaz$ is invoked. The variables $i$,
$j$, $x$ and $y$, and the constant symbol $\speccst$ are replaced by
$\{ a,b \}$.  After obvious simplifications, we obtain the following
set of instances:

\[
\begin{tabular}{c}
$a \leq b \vee \select(t,a) \leq \select(t,b)$ \\
$b \leq a \vee \select(t,b) \leq \select(t,a)$ \\
$\select(t,a) \leq \select(t,a) \vee \select'(t',\select(t,a)) \leq' \select'(t',\select(t,a))$ \\
$\select(t,a) \leq \select(t,b) \vee \select'(t',\select(t,a)) \leq' \select'(t',\select(t,b))$ \\
$\select(t,b) \leq \select(t,b) \vee \select'(t',\select(t,b)) \leq' \select'(t',\select(t,b))$ \\
$\select(t,b) \leq \select(t,a) \vee \select'(t',\select(t,a)) \leq' \select'(t',\select(t,a))$ \\
$a \leq b$\\
$\select'(t',\select(t,a)) > \select'(t',\select(t,b))$
\end{tabular}
\]

At this point, $\leq'$ may be simply replaced by $\leq$ (this
operation obviously preserves equisatisfiability) and the resulting clause
set can be refuted by any SMT-solver handling ground equality and
integer arithmetic.
\end{example}

Such nested array reads are outside the scope of the Array property
fragment of \cite{Bradleybook} and of the Logic LIA of
\cite{DBLP:conf/fossacs/HabermehlIV08}. They are not subsumed either by the extensions of the theory of arrays considered in \cite{GNRZ07}. Note that, due to the fact that we use distinct
renamings of the \theory\ of integers, equations such as
$\select(t',\select(t,a)) \iseq \select(t',a)$ are forbidden (if
arrays are viewed as heaps, this means that there can be no equation
between pointers and referenced values).

\section{Discussion}

\label{sect:disc}

In this paper we have introduced a new combination method of
instantiation schemes and presented sufficient conditions that
guarantee the completeness of the resulting instantiation scheme. As
evidenced by the examples provided in Section \ref{sect:appl}, this
combination method permits to obtain instantiation procedures for
several theories that are quite expressive, at almost no cost. One
direct consequence of these results is that it should be possible for
developers of SMT solvers to focus on the design of efficient decision
procedures for a few \emph{basic} theories, such as, e.g., the theory
of equality with uninterpreted function symbols (EUF) or Presburger
arithmetic, and obtain efficient SMT solvers for a large panel of
theories. 

This combination method may seem inefficient, since exponentially many
ground clauses may be generated, except for the trivial cases.  An
interesting line of research is to investigate how incremental
techniques can be implemented and the instantiations controlled so
that the (un)satisfiability of the clause set under consideration can
be detected before all clauses are instantiated in all possible
ways. For instance, we believe it is possible -- but this will
probably depend on $\basetheory$ and $\nonbasetheory$ -- to devise
more subtle strategies that begin by replacing {\bvariable}s with the
constants $\diam_\asort$ and applying the instantiation procedure for
$\nonbasetheory$, and deriving additional information from the
resulting set of ground clauses to avoid having to instantiate all
{\bvariable}s in all possible ways. Further investigations into this
line of work could lead to the design of more powerful instantiation
procedures that could enlarge the scope of modern SMT solvers by making them able to
handle efficiently
more expressive classes of quantified formul{\ae}.

\bibliography{biblio,Nicolas.Peltier}

\begin{thebibliography}{10}

\bibitem{Abadi2010153}
A.~Abadi, A.~Rabinovich, and M.~Sagiv.
\newblock Decidable fragments of many-sorted logic.
\newblock {\em Journal of Symbolic Computation}, 45(2):153 -- 172, 2010.

\bibitem{AlthausKW09}
E.~Althaus, E.~Kruglov, and C.~Weidenbach.
\newblock Superposition modulo linear arithmetic sup(la).
\newblock In S.~Ghilardi and R.~Sebastiani, editors, {\em FroCoS 2009}, volume
  5749 of {\em LNCS}, pages 84--99. Springer, 2009.

\bibitem{BG94}
L.~Bachmair and H.~Ganzinger.
\newblock Rewrite-based equational theorem proving with selection and
  simplification.
\newblock {\em Journal of Logic and Computation}, 3(4):217--247, 1994.

\bibitem{BachmairGW94}
L.~Bachmair, H.~Ganzinger, and U.~Waldmann.
\newblock Refutational theorem proving for hierachic first-order theories.
\newblock {\em Appl. Algebra Eng. Commun. Comput.}, 5:193--212, 1994.

\bibitem{Baumgartner:Tinelli:ModelEvolutionCalculus:CADE:2003}
P.~Baumgartner and C.~Tinelli.
\newblock {The Model Evolution Calculus}.
\newblock In F.~Baader, editor, {\em CADE-19 -- The 19th International
  Conference on Automated Deduction}, volume 2741 of {\em LNAI}, pages
  350--364. Springer, 2003.

\bibitem{Bradleybook}
A.~R. Bradley and Z.~Manna.
\newblock {\em The Calculus of Computation: Decision Procedures with
  Applications to Verification}.
\newblock Springer-Verlag New York, Inc., Secaucus, NJ, USA, 2007.

\bibitem{Bradley}
A.~R. Bradley, Z.~Manna, and H.~B. Sipma.
\newblock What's decidable about arrays?
\newblock In E.~A. Emerson and K.~S. Namjoshi, editors, {\em Proc. VMCAI-7},
  volume 3855 of {\em LNCS}, pages 427--442. Springer, 2006.

\bibitem{DBLP:journals/amai/BruttomessoCFGS09}
R.~Bruttomesso, A.~Cimatti, A.~Franz{\'e}n, A.~Griggio, and R.~Sebastiani.
\newblock Delayed theory combination vs. nelson-oppen for satisfiability modulo
  theories: a comparative analysis.
\newblock {\em Ann. Math. Artif. Intell.}, 55(1-2):63--99, 2009.

\bibitem{tata97}
H.~Comon, M.~Dauchet, R.~Gilleron, F.~Jacquemard, D.~Lugiez, S.~Tison, and
  M.~Tommasi.
\newblock Tree automata techniques and applications.
\newblock Available on: {\tt http://www.grappa.univ-lille3.fr/tata}, 1997.

\bibitem{CD94}
H.~Comon and C.~Delor.
\newblock Equational formulae with membership constraints.
\newblock {\em Information and Computation}, 112(2):167--216, August 1994.

\bibitem{DG79}
B.~Dreben and W.~D. Goldfarb.
\newblock {\em The Decision Problem, Solvable Classes of Quantificational
  Formulas}.
\newblock Addison-Wesley, 1979.

\bibitem{EP10a}
M.~Echenim and N.~Peltier.
\newblock {Instantiation of SMT problems modulo Integers}.
\newblock In {\em AISC 2010 (10th International Conference on Artificial
  Intelligence and Symbolic Computation)}, LNCS. Springer, 2010.

\bibitem{EP10c}
M.~Echenim and N.~Peltier.
\newblock An instantiation scheme for satisfiability modulo theories.
\newblock {\em Journal of Automated Reasoning}, 2010.

\bibitem{DBLP:conf/frocos/Fontaine09}
P.~Fontaine.
\newblock Combinations of theories for decidable fragments of first-order
  logic.
\newblock In S.~Ghilardi and R.~Sebastiani, editors, {\em FroCos}, volume 5749
  of {\em Lecture Notes in Computer Science}, pages 263--278. Springer, 2009.

\bibitem{DBLP:conf/lpar/FontaineRZ04}
P.~Fontaine, S.~Ranise, and C.~G. Zarba.
\newblock Combining lists with non-stably infinite theories.
\newblock In F.~Baader and A.~Voronkov, editors, {\em LPAR}, volume 3452 of
  {\em Lecture Notes in Computer Science}, pages 51--66. Springer, 2004.

\bibitem{DBLP:conf/lics/Ganzinger01}
H.~Ganzinger.
\newblock Relating semantic and proof-theoretic concepts for polynominal time
  decidability of uniform word problems.
\newblock In {\em LICS}, pages 81--92, 2001.

\bibitem{GanzingerKorovin-03-lics}
H.~Ganzinger and K.~Korovin.
\newblock New directions in instantiation-based theorem proving.
\newblock In {\em Proc.\ 18th IEEE Symposium on Logic in Computer
  Science,(LICS'03)}, pages 55--64. IEEE Computer Society Press, 2003.

\bibitem{dMGe}
Y.~Ge and L.~M. de~Moura.
\newblock Complete instantiation for quantified formulas in satisfiabiliby
  modulo theories.
\newblock In A.~Bouajjani and O.~Maler, editors, {\em CAV 2009}, volume 5643 of
  {\em LNCS}, pages 306--320. Springer, 2009.

\bibitem{springerlink:10.1007/s10472-007-9078-x}
S.~Ghilardi, E.~Nicolini, S.~Ranise, and D.~Zucchelli.
\newblock Decision procedures for extensions of the theory of arrays.
\newblock {\em Annals of Mathematics and Artificial Intelligence}, 50:231--254,
  2007.
\newblock 10.1007/s10472-007-9078-x.

\bibitem{GNRZ07}
S.~Ghilardi, E.~Nicolini, S.~Ranise, and D.~Zucchelli.
\newblock Decision procedures for extensions of the theory of arrays.
\newblock {\em Ann. Math. Artif. Intell.}, 50(3-4):231--254, 2007.

\bibitem{DBLP:conf/fossacs/HabermehlIV08}
P.~Habermehl, R.~Iosif, and T.~Vojnar.
\newblock What else is decidable about integer arrays?
\newblock In R.~M. Amadio, editor, {\em FoSSaCS}, volume 4962 of {\em Lecture
  Notes in Computer Science}, pages 474--489. Springer, 2008.

\bibitem{LP92}
S.~Lee and D.~A. Plaisted.
\newblock Eliminating duplication with the hyper-linking strategy.
\newblock {\em Journal of Automated Reasoning}, 9:25--42, 1992.

\bibitem{DBLP:journals/cj/LoosW93}
R.~Loos and V.~Weispfenning.
\newblock Applying linear quantifier elimination.
\newblock {\em Comput. J.}, 36(5):450--462, 1993.

\bibitem{PZ00}
D.~A. Plaisted and Y.~Zhu.
\newblock Ordered semantic hyperlinking.
\newblock {\em Journal of {A}utomated {R}easoning}, 25(3):167--217, October
  2000.

\bibitem{DBLP:conf/cade/Sofronie-Stokkermans05}
V.~Sofronie-Stokkermans.
\newblock Hierarchic reasoning in local theory extensions.
\newblock In R.~Nieuwenhuis, editor, {\em CADE}, volume 3632 of {\em Lecture
  Notes in Computer Science}, pages 219--234. Springer, 2005.

\bibitem{DBLP:conf/cade/Sofronie-Stokkermans10}
V.~Sofronie-Stokkermans.
\newblock Hierarchical reasoning for the verification of parametric systems.
\newblock In J.~Giesl and R.~H{\"a}hnle, editors, {\em IJCAR}, volume 6173 of
  {\em Lecture Notes in Computer Science}, pages 171--187. Springer, 2010.

\bibitem{Tinelli96anew}
C.~Tinelli and M.~Harandi.
\newblock {A new correctness proof of the Nelson-Oppen combination procedure}.
\newblock In {\em Frontiers of Combining Systems, volume 3 of Applied Logic
  Series}, pages 103--120. Kluwer Academic Publishers, 1996.

\end{thebibliography}
\arttocl{\bibliographystyle{acmtrans}}
\resrep{\bibliographystyle{abbrv}}

\end{document}